\pdfoutput=1
\documentclass{article}

\PassOptionsToPackage{numbers}{natbib}
\usepackage[main, final]{neurips_2026}

\usepackage[utf8]{inputenc}%
\usepackage[T1]{fontenc}%
\usepackage{hyperref}%
\usepackage{url}%
\usepackage{booktabs}%
\usepackage{amsfonts}%
\usepackage{nicefrac}%
\usepackage{microtype}%
\usepackage{xcolor}%
\usepackage{amsmath}
\usepackage{algorithm}
\usepackage{algorithmic}
\usepackage{thmtools}
\usepackage{enumitem}
\usepackage{graphicx}%
\usepackage{subcaption}%
\usepackage{float}
\usepackage{placeins}%
\newcommand{\oom}{{\color{gray}\textsc{OOM}}}

\usepackage{amsthm}

\newcommand{\G}{\mathcal{G}}%
\newcommand{\V}{\mathcal{V}}%
\newcommand{\E}{\mathcal{E}}%

\newtheorem{remark}{Remark}
\title{Graph Anomaly Detection as Finite-Horizon Control:\\Training-Free Scoring via Empirical Bayes}

\author{%
  Fred Xu\thanks{Work done during an internship at Block, Inc.} \\
  University of California, Los Angeles \\
  \texttt{fredxu@cs.ucla.edu} \\
  \And
  Thomas Markovich \\
  Block \\
  \texttt{tmarkovich@squareup.com} \\
  \AND
  Florence Regol \\
  Mila -- Quebec AI Institute \\
  \texttt{florence.regol@mila.quebec} \\
  \And
  Yizhou Sun \\
  University of California, Los Angeles \\
  \texttt{yzsun@cs.ucla.edu} \\
}

\begin{document}

\maketitle

\begin{abstract}
Node-level graph anomaly detection (GAD) identifies nodes whose attributes and interactions deviate from dominant graph regularities. Existing GAD models encode normality and anomaly scoring indirectly through architectures, message passing, reconstruction or contrastive objectives, and tuned score families. This entangles graph trust (how strongly graph structure should define normality), graph-spectral weighting, and anomaly-score choice, yielding scores that are costly, opaque, and unstable across graph regimes. We propose EB-GAD (Empirical-Bayes GAD), a training-free framework that models normality as graph-aware generalized Ornstein--Uhlenbeck (GOU) relaxation toward a graph-filtered template. Empirical Bayes fits the graph precision from the residual-field likelihood; the GOU then turns scoring into a closed-form finite-horizon control energy, the minimum effort to steer a feature-neutral node to its observed endpoint along graph-spectral relaxation. Sweeping relaxation horizon and endpoint tolerance yields a bank of scores that share one fitted prior: equilibrium Mahalanobis scoring is one limit, while finite-horizon control-energy and scale-normalized ratio scores reveal anomalies that static equilibrium scoring can mask. A label-free selector chooses the score family from feature homophily, edge density, and feature dimension, then ranks candidates by fitted-null deviation and rank stability. On 11 benchmarks and without labels at any step, EB-GAD has the best or tied-best AUROC on 9: the four financial fraud networks (up to 3.7M nodes), the YelpChi and Amazon review graphs, Weibo, Reddit and Facebook, with margins of up to 21.7 points. It is second on BlogCatalog and ACM.
\end{abstract}

\section{Introduction}

Node-level graph anomaly detection (GAD) asks which nodes look abnormal relative to both their attributes and graph neighborhoods, with applications in fraud detection, cybersecurity, and scientific networks \cite{akoglu2015graph,ma2021comprehensive}. This is difficult without labels because edges are not always equally trustworthy: in homophilic graphs they reveal normal communities, while in fraud, camouflage, or heterophilic settings they can hide abnormal behavior \cite{liu2020graphconsis,dou2020caregnn,liu2021pcgnn,wang2023gaga,xiang2023gtan,xu2024revisiting}. A detector must therefore decide how much graph structure should define normality (graph trust), which graph-scale variations matter, and how to turn deviations into an anomaly score.

Most GAD methods make these choices indirectly through architectures, message passing, reconstruction or contrastive losses, and hyperparameter sweeps \cite{ding2019deep,li2019specae,fan2020anomalydae,Liu_2022_cola,jin2021anemone,zheng2021slgad,li2024diffgad,qiao2023truncated}. This makes scores hard to interpret and unstable across graph regimes: in native-score evaluation, several baselines fall below chance on Elliptic and Facebook (Table~\ref{tab:main_auroc}). It also leaves no way to ask whether a node is anomalous because it is graph-rough, slow to match a normal template, or large in residual magnitude.

EB-GAD makes these choices explicit. We fit a structured prior in which normal features relax toward a graph-filtered template, while empirical Bayes (EB) selects graph trust and length-scale from the residual-field likelihood \cite{robbins1956empirical,efron2012large}; since anomalies are rare, this likelihood is dominated by normal nodes. The prior is the stationary law of a graph-aware generalized Ornstein--Uhlenbeck (GOU) process \cite{yue2023goub} with graph Mat\'ern precision \cite{borovitskiy2021materngaussianprocessesgraphs,xu2025uncertaintyestimationgraphsstructure}. This connects modeling to scoring: starting from neutral features, how much control effort is needed to reach the observed node as graph modes relax? The answer is a \emph{finite-horizon control energy}, available in closed form. Varying the relaxation horizon $\Gamma$ and endpoint tolerance $\lambda_c$ gives a \emph{bank} of such scores. The bank keeps one fitted prior but tests different views: short-horizon effort, equilibrium roughness, and scale-normalized roughness.

We call the resulting framework EB-GAD (Empirical-Bayes GAD) and describe it as \emph{training-free}. By this we mean that no neural parameters are learned, no gradient descent or epochs are run, and no labels or validation sets are used. EB-GAD does fit three scalars $(\rho,\kappa,\gamma)$, by maximizing a closed-form likelihood on the unlabeled graph (a clamped Newton iteration in $\rho$, small grids over $\kappa$ and $\gamma$), and it ranks a fixed bank of closed-form scores with a label-free selector whose grids and constants are fixed in advance (Appendix~\ref{app:selector_details}). Training-free does not mean hyperparameter-free; a longer name is closed-form scoring with empirical-Bayes calibration. Each ingredient (graph Gaussian priors, likelihood-based fitting, OU processes, Mahalanobis scores) is established; our contribution is their calibrated composition and the evidence that it works without training:

\begin{enumerate}[nosep,leftmargin=*]
\item \textbf{Anomaly score as finite-horizon control energy.}
We extend equilibrium graph-Mahalanobis scoring to a finite-horizon bank using the GOU's closed-form Gaussian transitions. The scores ($C_{\Gamma,\lambda_c}$, $CR_{\Gamma,\lambda_c}$, $J^\star$, $R$) share one EB-fitted graph precision, hence one plug-in null; $J^\star$ is a limiting case, not the whole model.

\item \textbf{Label-free selection over the bank.}
A predeclared selector chooses among equilibrium, control-energy, control-ratio, and profile scores without labels. Graph statistics pick the family; within it, fitted-null deviation and neighboring-rank stability pick a stable score rather than an isolated AUROC-favored configuration.

\item \textbf{Training-free empirical Bayes prior selection.}
The linear-Gaussian structure makes the residual-space likelihood closed-form. Algorithm~\ref{alg:prior_opt} fits the prior by scalar optimizations; its parameters are interpretable ($\rho$: graph trust; $\kappa$: inverse length-scale; $\gamma$: template bandwidth).
\end{enumerate}

On 11 benchmarks and without labels at any step, EB-GAD has the best or tied-best AUROC on 9: the four financial fraud networks (up to 3.7M nodes), the YelpChi and Amazon review graphs, Weibo, Reddit and Facebook, with margins of up to 21.7 points. It is second on BlogCatalog and ACM. Section~\ref{sec:experiments} and Section~\ref{sec:conclusion} state where it trails, and which entries depend on a design choice made for this version.

\begin{figure*}[t]
    \centering
    \includegraphics[width=\linewidth]{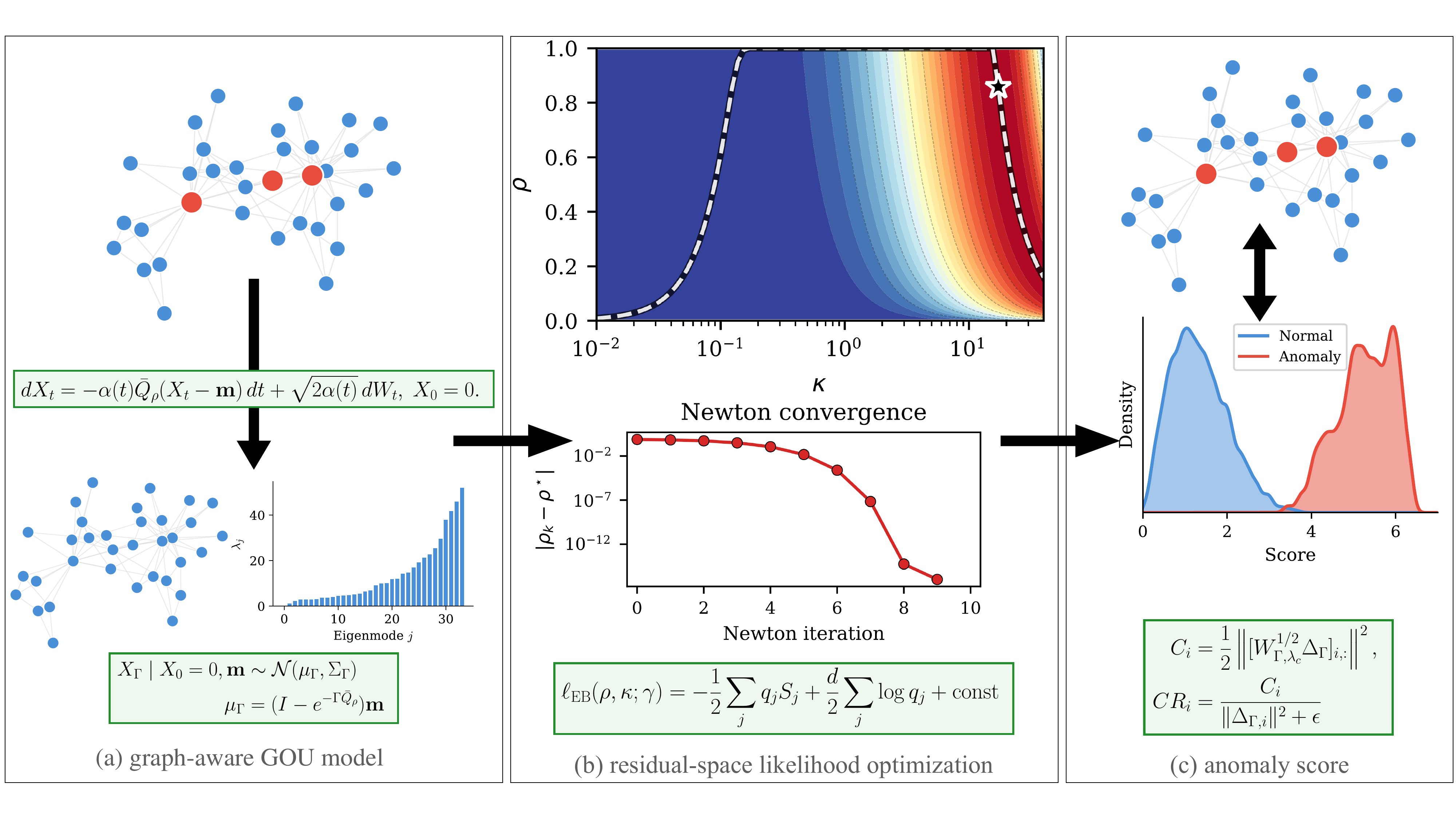}
    \caption{%
        \textbf{EB-GAD overview.}
        \textbf{(a)} A graph-aware GOU around template $\mathbf m$ defines finite-horizon Gaussian transitions from a feature-neutral start.
        \textbf{(b)} Residual-space empirical Bayes fits $(\rho,\kappa)$ for each template bandwidth $\gamma$ without labels ($d$ in (b) is the feature dimension $D$).
        \textbf{(c)} The fitted dynamics produce control-energy scores $C_{\Gamma,\lambda_c}$, $CR_{\Gamma,\lambda_c}$ and equilibrium scores $J^\star,R$; a fixed selector routes each dataset to one of four score families.
    }
    \label{fig:overview}
\end{figure*}

\section{Background}

\subsection{Graph Anomaly Detection}
\label{subsec:gad_background}

Graph anomaly detection (GAD) identifies atypical nodes whose attributes and structural interactions deviate from graph regularities. We focus on \emph{node-level} GAD on attributed graphs \cite{akoglu2015graph,ma2021comprehensive}. Formally, given $\G=(\V,\E,X)$ with $|\V|=n$, adjacency $A$, and node features $X\in\mathbb{R}^{n\times D}$, unsupervised node-level GAD outputs scores $s_i=f(i,A,X)$ whose large values rank nodes $i\in\V$ as anomalous. Existing methods span reconstruction/generative approaches \cite{ding2019deep,li2019specae,fan2020anomalydae,li2024diffgad}, self-supervised/contrastive methods \cite{Liu_2022_cola,jin2021anemone,zheng2021slgad}, and one-class methods \cite{qiao2023truncated}. In fraud settings, they must also handle class imbalance, camouflage, and heterophily \cite{liu2020graphconsis,dou2020caregnn,liu2021pcgnn,wang2023gaga,xiang2023gtan,xu2024revisiting}. Across these families, graph trust, graph-spectral weighting, and anomaly scoring are often tied to architectures and hyperparameter sweeps.
EB-GAD instead casts these decisions as parameters of a single linear-Gaussian model and selects them from data without labels.

\subsection{Empirical Bayes for Anomaly Detection}
\label{subsec:bayes_background}

Unsupervised anomaly detection can be cast as evaluating \emph{Shannon surprise} under a normal-data model, using $s(x_i) = -\log p(x_i\mid\varphi)$. We use $\varphi$ for prior or marginal-model hyperparameters, not arbitrary trained weights; in EB-GAD these are graph trust, graph length-scale, and template bandwidth. Many methods approximate this score through density estimation \cite{lof_2000}, reconstruction error \cite{ding2019deep,fan2020anomalydae}, or representation-space distance \cite{qiao2023truncated,Liu_2022_cola}; in each case, the fitted normality notion determines the anomaly score.
In \emph{empirical Bayes} \cite{robbins1956empirical,efron2012large}, these hyperparameters are selected from data $\mathcal{D}$ by maximizing the marginal likelihood:
\begin{equation}
\varphi^\star = \arg\max_\varphi\; \log p(\mathcal{D}\mid\varphi),
\label{eq:eb_bg}
\end{equation}

which integrates out latent variables when present and balances fit against complexity (Occam's razor) \cite{efron2012large}.
Since anomalies are rare, this criterion is dominated by normal nodes and $\varphi^\star$ approximately recovers a normal-data prior. When the model is \emph{conjugate} (Gaussian prior with linear-Gaussian likelihood), the marginal likelihood is available in closed form; otherwise one typically optimizes a variational lower bound (ELBO). In our setting, empirical Bayes is applied to the residual field induced by a low-pass graph-filtered template rather than directly to the raw features. The likelihood is a graph-coupled density over the full residual field, while the node score is the nodewise contribution to the same joint Gaussian energy. Bayesian outlier detection has been explored in non-graph settings \cite{dpmmoutlier2011,ting2007bayesoutlier}; in graph anomaly detection, prior structure is usually fixed rather than selected from data.

\subsection{Generalized Ornstein--Uhlenbeck Dynamics}
\label{subsec:gou_background}

We need a normal-data model that can encode prior knowledge about where node features should concentrate, while remaining tractable enough for label-free fitting. The Ornstein--Uhlenbeck process provides exactly this: a linear Gaussian model whose stationary mean can be set to a template $\mathbf m$. We use its matrix-drift generalization,
\begin{equation}
dX_t = -\alpha(t)\,\bar Q_\rho\,(X_t - \mathbf{m})\,dt + \sqrt{2\alpha(t)}\,dW_t,
\label{eq:gou_background}
\end{equation}
where $\alpha(t)>0$ is a scalar schedule, $W_t$ is Brownian motion, and $\bar Q_\rho$ is a symmetric positive definite (SPD) drift operator. Coupling the noise scale to $\alpha(t)$ gives stationary distribution $\mathcal{N}(\mathbf m,\bar Q_\rho^{-1})$ and closed-form Gaussian transitions. Section~\ref{subsec:graph_prior} instantiates $\mathbf m$ as a graph-filtered template and $\bar Q_\rho$ as a graph-aware precision, injecting graph prior knowledge into both the stationary mean and covariance while keeping EB fitting and node scoring closed form.

\section{Graph Anomaly Detection via Finite-Horizon Control}

The components of EB-GAD are standard: a graph Gaussian prior, the OU dynamics attached to it, and Gaussian quadratic scores. The contribution is their composition, calibrated by EB and selected without labels. We therefore state the model first and its control reading afterwards: \S\ref{subsec:graph_prior}--\ref{subsec:anomaly_score} define a graph-aware GOU and derive its bank of finite-horizon scores; \S\ref{subsec:prior_selection} fits the prior and selects a score.

\subsection{Graph-Aware GOU Process}
\label{subsec:graph_prior}

\paragraph{Notation.}
We use three forms of the same node-feature data: input matrix $X\in\mathbb{R}^{n\times D}$, its row-major vectorization $\bar x=\operatorname{vec}(X)\in\mathbb{R}^{nD}$, and the GOU state $X_t\in\mathbb{R}^{nD}$ (with $X_0=0$ and $X_T=\bar x$ at scoring time; lowercase $x_T$ in SOC formulae denotes the same realization). The objects $X,M,\Delta,\Delta_\Gamma$ are used in either $n\times D$ matrix or $\mathbb{R}^{nD}$ vec form depending on context (covariance equations use the vec form; nodewise scoring uses the matrix); $[A]_{i,:}\in\mathbb{R}^D$ denotes the $i$-th row of an $n\times D$ matrix.

\paragraph{Reference dynamics.}
Let $W_t \in \mathbb{R}^{nD}$ be a standard Brownian motion. We define a graph-aware GOU reference process around a template $\mathbf{m}\in\mathbb{R}^{nD}$:
\begin{equation}
dX_t
=
-\alpha(t)\,\bar Q_{\rho}\,\bigl(X_t - \mathbf{m}\bigr)\,dt
+
\sqrt{2\alpha(t)}\,dW_t,
\qquad t\in[0,T],
\label{eq:ref_graph_ou}
\end{equation}
where $\alpha(t)>0$ is a deterministic scalar schedule and $\bar{Q}_\rho$ is a graph-aware precision operator.
The matrix drift $\bar Q_\rho$ embeds the graph into the dynamics: in eigencoordinates of $\bar Q_\rho$, mode $j$ relaxes at rate $\alpha(t)q_j$, so smooth (low Laplacian frequency, small $q_j$) modes drift slowly toward the template while rough modes equilibrate fast.
The closed-form Gaussian transition that follows therefore inherits a graph-spectral schedule of relaxation times.

\paragraph{Role of the dynamics for static data.}
We do not claim that the data evolved by OU dynamics, and static data could not support such a claim. The choice is representational. Among linear diffusions with isotropic noise, \eqref{eq:ref_graph_ou} is the unique one, up to time scale, that is reversible with stationary law $\mathcal N(\mathbf m,\bar Q_\rho^{-1})$; it is the canonical dynamics of the graph Gaussian prior that EB fits and adds no modeling freedom at equilibrium. What it adds is a horizon-indexed family of closed-form quadratic scores that share this one prior, so that every horizon is calibrated by the same fitted law; a heat or polynomial filter bank changes the implied prior at each scale. The Brownian increments are independent of the current state by definition of the linear-Gaussian transition, not as an assumption on the data: a state-dependent innovation would be absorbed into a different drift and stationary precision, which is what $(\rho,\kappa)$ fit, and since only the stationary law and its finite-horizon interpolations enter the scores, such dependence is not identifiable from a static graph.

\paragraph{Precision blend: graph trust and inverse length-scale.}
We parameterize the precision operator with a graph trust parameter $\rho\in[0,1]$ and Mat\'ern inverse length-scale $\kappa \ge 0$:
\begin{equation}
\bar Q_{\rho}:=Q_{\rho}\otimes I_D,
\qquad
Q_{\rho}:=\rho\,Q_{\mathrm{prior}}+(1-\rho)\,I_n,
\qquad
Q_{\mathrm{prior}}=(\kappa^2 I_n+ L)^\nu,
\label{eq:Q_blend}
\end{equation}
where $\otimes$ denotes the Kronecker product, $\nu>0$ controls the spectral roll-off, i.e., how quickly precision grows with graph frequency (fixed to $\nu=1$ throughout), and $L$ is the Laplacian of the selected graph operator (the input graph by default, or a predeclared feature-affinity reweighting that keeps the same edge set; Appendix~\ref{app:selector_details}). In eigenspace, the precision on mode $j$ is $q_j = \rho(\kappa^2+\lambda_j)^\nu + (1-\rho)$, where $\lambda_j$ is the $j$-th selected-operator eigenvalue. The two parameters play distinct roles:
$\rho$ controls the overall \emph{graph trust}: $\rho=1$ gives full graph-spectral precision, while $\rho=0$ gives isotropic precision. $\kappa$ is the Mat\'ern inverse length-scale: modes with $\lambda_j \gg \kappa^2$ are dominated by graph-frequency weighting, while modes with $\lambda_j \ll \kappa^2$ share the approximately constant precision level $(1-\rho)+\rho\kappa^{2\nu}$. Thus $\kappa$ controls where the transition from low-frequency floor to graph-spectral growth occurs, rather than forcing the low-frequency precision to equal the isotropic value one. This follows from Mat\'ern kernels on graphs \cite{borovitskiy2021materngaussianprocessesgraphs, xu2025uncertaintyestimationgraphsstructure}.

\paragraph{Prior template.}
The template $\mathbf m\in\mathbb{R}^{nD}$ encodes a prior notion of normal node attributes, defined as a low-pass graph-filtered version of the features:
\begin{equation}
\mathbf m
:=\bar S_{\gamma,\nu}\,\bar x,
\qquad
S_{\gamma,\nu}:=(\gamma^2 I_n+ L)^{-\nu},
\qquad \gamma>0,\ \nu>0,
\label{eq:template}
\end{equation}
where $\bar S_{\gamma,\nu}:=S_{\gamma,\nu}\otimes I_D$ is the unnormalized Mat\'ern resolvent used in our implementation \cite{borovitskiy2021materngaussianprocessesgraphs,xu2025uncertaintyestimationgraphsstructure}.
It is low-pass in graph frequency; because it is unnormalized, the profiled $\gamma$ controls both bandwidth and low-frequency scale, and eigenvalues $1/(\gamma^2+\lambda_j)^\nu$ can exceed one when $\gamma^2+\lambda_j<1$ (e.g., DC gain $\gamma^{-2\nu}$ on a Laplacian zero mode). We treat this as a fixed template family rather than a post-hoc normalization: EB fitting and scoring always use the same chosen residual operator, while a unit-DC-gain variant would simply define a different predeclared template family.
The low-pass construction is canonical, while heterophilic high-dimensional regimes may use the same Mat\'ern construction with a predeclared affinity-weighted template without changing the Gaussian or control-energy derivation.
Define $T_{\gamma,\nu}:=I_n-S_{\gamma,\nu}$ and $\Delta=T_{\gamma,\nu}X$.
All likelihoods and scores are computed on this residual field rather than on raw features; truncated eigenspace and null-mode details are given in Appendix~\ref{app:modeling_details}.

\paragraph{Transition kernel.}
For $0\le s < t \le T$, define the integrated schedule $\Gamma_{s:t}:=\int_s^t\alpha(r)\,dr$ and the propagator $\Phi_{s:t}:=\exp(-\Gamma_{s:t}\bar Q_\rho)$.

\begin{restatable}{proposition}{reftransition}\label{prop:ref_transition_score}
Under \eqref{eq:ref_graph_ou}, the transition kernel $p(X_t\mid X_s)$ is Gaussian with mean
$\mu_{t|s}(x_s) = \mathbf m + \Phi_{s:t}(x_s-\mathbf m)$ and covariance
$\Sigma_{t|s} = \bar Q_\rho^{-1}(I_{nD}-\exp(-2\Gamma_{s:t}\bar Q_\rho))$ whenever $q_j>0$ on all modeled modes (equivalently, $\bar Q_\rho$ is SPD on the modeled subspace).
\end{restatable}

\noindent In eigenspace, the covariance on mode $j$ at time $\Gamma$ is $\sigma_j^2(\Gamma) = (1-e^{-2\Gamma q_j})/q_j$, which grows from $0$ (at $\Gamma=0$) toward the equilibrium $1/q_j$ (as $\Gamma\to\infty$). The rate of growth is governed by $q_j$: high-frequency modes ($q_j$ large) reach equilibrium quickly, while low-frequency modes relax slowly. The proof is given in Appendix~\ref{sec:proofs}.

\subsection{Finite-Horizon Generative Model and Control Energy}
\label{subsec:generative_model}

We score a node by the dynamical effort needed to generate its observed features from a feature-neutral start under the graph-aware GOU.
Concretely: imagine steering the GOU from the origin to the observed endpoint $\bar x$ at time $\Gamma$; the minimum steering energy (under hard endpoint matching) is exactly the negative log-density of the GOU's unconditioned endpoint distribution, up to a constant.
This identification turns the closed-form Gaussian transition kernel into an anomaly score, the \emph{finite-horizon control energy}, and it lets a single fitted GOU yield a bank of scores by varying the horizon $\Gamma$ and the endpoint tolerance $\lambda_c$. It is the minimum control energy of a GOU bridge to one observed endpoint: no coupling or transport plan is optimized, so it is not an optimal-transport cost.

\paragraph{Generative model.}
We model normal data as endpoints of an unconditioned GOU process mean-reverting toward $\mathbf m$ after effective duration $\Gamma:= \int_0^T \alpha(r)\,dr$, started at the feature-neutral origin $X_0=0$ (justified because features are z-scored).
This gives predictive mean $\mu_\Gamma=(I_{nD}-e^{-\Gamma\bar Q_\rho})\mathbf m$, which converges to $\mathbf m$ as $\Gamma\to\infty$.
Throughout this subsection, we assume $q_j>0$ on every modeled mode, equivalently that $\bar Q_\rho$ is SPD on the modeled subspace.
For a fixed template (equivalently, fixed $\gamma$), Proposition~\ref{prop:ref_transition_score} gives the residual-space law
\begin{equation}
p(\Delta_\Gamma \mid \varphi, \Gamma)
= \mathcal{N}\!\bigl(\Delta_\Gamma;\; 0,\; \Sigma_\Gamma\bigr),
\qquad
\Delta_\Gamma := X_\Gamma - \mu_\Gamma,
\qquad
\Sigma_{\Gamma} = \bar Q_\rho^{-1}(I_{nD} - e^{-2\Gamma \bar Q_\rho}).
\label{eq:finite_time_predictive}
\end{equation}
Equivalently, conditional on the template and neutral start, $X_\Gamma \sim \mathcal{N}(\mu_\Gamma,\Sigma_\Gamma)$; at scoring time this endpoint is the observed $\bar x$, so the realized residual is $\Delta_\Gamma=\bar x-\mu_\Gamma$.
Normal residuals are plausible samples from this process; anomalies are not.

\paragraph{Time-dependent precision.}
The precision of the finite-horizon distribution in eigenspace is
\begin{equation}
\tilde q_j(\Gamma)
:= \Sigma_{\Gamma,j}^{-1}
= \frac{q_j}{1 - e^{-2\Gamma q_j}} \;\ge\; q_j,
\label{eq:time_dep_precision}
\end{equation}
with equality in the limit $\Gamma \to \infty$.

\begin{restatable}{proposition}{equilibriumopt}\label{prop:equilibrium_opt}
\textbf{(Equilibrium convergence.)}
For any fixed prior parameters $(\rho, \kappa, \gamma)$ such that $q_j>0$ on all modeled modes, the time-dependent precision satisfies $\tilde q_j(\Gamma) \ge q_j$ for all $j$ and all finite $\Gamma>0$, with exponential convergence:
\begin{equation}
\frac{\tilde q_j(\Gamma)}{q_j} = \frac{1}{1-e^{-2\Gamma q_j}} = 1 + O(e^{-2\Gamma q_j}).
\label{eq:precision_convergence}
\end{equation}
Consequently, for any fixed residual object $Y$, the finite-horizon and equilibrium precision-weighted residual-space likelihood objectives $\ell_\Gamma(Y;\varphi)$ and $\ell_\infty(Y;\varphi)$ satisfy
$\ell_\Gamma(Y;\varphi) \to \ell_\infty(Y;\varphi)$
as $\Gamma \to \infty$.
Moreover, for $\Gamma q_{\min}\ge \tfrac12\log 2$ with $q_{\min}=\min_{1\le j\le k} q_j$, the objective difference is bounded by
\[
\bigl|\ell_\Gamma(Y;\varphi)-\ell_\infty(Y;\varphi)\bigr|
\le C(Y,\varphi)\,e^{-2\Gamma q_{\min}},
\qquad
C(Y,\varphi):=\sum_{j=1}^{k}(q_j S_j(Y)+D).
\]
Here $S_j(Y)$ is the spectral energy of $Y$ on retained mode $j$.
\end{restatable}

\noindent At the optimized $(\rho^\star,\kappa^\star)$, $q_{\min}>0$ on every retained mode (because $\rho^\star<1$ or $\kappa^\star>0$ in every selected configuration), so the bound is exponentially tight; the endpoint residual used for scoring satisfies the same $O(e^{-\Gamma q_{\min}})$ convergence (Step~4 of the proof in Appendix~\ref{sec:proofs}). The final method deliberately retains finite-horizon scores before this collapse, because different graph regimes expose their anomalous signal at different relaxation horizons.

\paragraph{Control energy interpretation.}
For a fixed template $\mathbf m$, the conditional Gaussian score induced by \eqref{eq:finite_time_predictive} equals, up to an $x_T$-independent constant, the minimum control energy in a stochastic optimal control (SOC) formulation, since both reduce to the quadratic form $\frac{1}{2}(x_T - \mu_\Gamma)^\top \Sigma_\Gamma^{-1}(x_T - \mu_\Gamma)$.
Consider the GOU with steering input:
\begin{equation}
dX_t = -\alpha(t)\,\bar Q_\rho(X_t - \mathbf{m})\,dt + g(t)\,u_t(X_t)\,dt + g(t)\,dW_t,
\qquad X_0 = 0,
\label{eq:steered_gou}
\end{equation}
with $g(t) = \sqrt{2\alpha(t)}$, where $u_t$ steers the neutral-start process toward the observed data $x_T$ while the drift mean-reverts toward the template $\mathbf m$. The uncontrolled endpoint mean is $\mu_\Gamma=(I_{nD}-e^{-\Gamma\bar Q_\rho})\mathbf m$. The control energy under hard endpoint matching ($\lambda_c \to \infty$) is
\begin{equation}
\mathcal{E}(0 \to x_T;\, \Gamma)
= \frac{1}{2}(x_T - \mu_\Gamma)^\top \Sigma_\Gamma^{-1} (x_T - \mu_\Gamma)
= -\log p(x_T \mid X_0=0,\mathbf m,\varphi, \Gamma) + \mathrm{const}.
\label{eq:transport_energy}
\end{equation}
This gives the score a dynamical interpretation: it measures the effort needed to steer the GOU to the observed features at horizon~$\Gamma$ from a neutral start under attraction to the template. The graph-aware precision operator makes spectrally rough deviations expensive and spectrally smooth ones cheap.

\subsection{Per-Node Anomaly Scoring}
\label{subsec:anomaly_score}

We now decompose the global Gaussian energy of \eqref{eq:finite_time_predictive} into per-node contributions, then generalize from hard endpoint matching to a finite-horizon bank parameterized by an endpoint tolerance.

\paragraph{Graph signal model.}
The Kronecker structure $\bar Q_\rho = Q_\rho \otimes I_D$ means the $D$ feature dimensions are conditionally independent given the graph.
At scoring time the GOU endpoint equals the observation, $X_T=\bar x$, so the realized finite-horizon residual is the matrix $\Delta_\Gamma\in\mathbb{R}^{n\times D}$ obtained by reshaping $\bar x-\mu_\Gamma$, with rows $\delta_{\Gamma,i}:=[\Delta_\Gamma]_{i,:}=[X]_{i,:}-[\mu_\Gamma]_{i,:}$, representing $D$ conditionally independent graph signals.
Its hard-endpoint Gaussian log-likelihood, up to the additive constant $-\tfrac{kD}{2}\log 2\pi$, is
\begin{equation}
\log p(\mathcal{D}\mid\varphi, \Gamma)
=
-\frac{1}{2}\sum_{j=1}^{k} \tilde q_j(\Gamma)\,S_j(\Gamma)
+\frac{D}{2}\sum_{j=1}^{k} \log \tilde q_j(\Gamma)
\label{eq:joint_ll}
\end{equation}
where $S_j(\Gamma) := \|\Delta_\Gamma^\top v_j\|^2$ and $k \le n$ is the number of computed eigenpairs.
Equation~\eqref{eq:joint_ll} is the hard-endpoint case ($\lambda_c\to\infty$) of a more general soft-endpoint Gaussian.
We retain the finite-horizon GOU transition instead of taking $\Gamma\to\infty$, because short horizons ask whether the node is surprising before the template has fully relaxed, while long horizons approach the equilibrium residual $\Delta=\bar x-\mathbf m$.

\paragraph{Finite-horizon score bank.}
Hard endpoint matching is a knife-edge: it asserts that the GOU sample at time $\Gamma$ exactly equals the observation $\bar x$. Allowing endpoint uncertainty $\lambda_c<\infty$ tests whether the anomaly evidence persists under softer matching, where small deviations from the noiseless endpoint are tolerated.
We model this by adding observation noise: if $Z_\Gamma\sim\mathcal{N}(\mu_\Gamma,\Sigma_\Gamma)$ and $\bar x=Z_\Gamma+\eta$ with $\eta\sim\mathcal{N}(0,\lambda_c^{-1}I)$, then marginalizing $Z_\Gamma$ gives $\bar x\sim\mathcal{N}(\mu_\Gamma,\Sigma_\Gamma+\lambda_c^{-1}I)$.
The effective finite-horizon precision on spectral mode $j$ is therefore
\begin{equation}
c_j(\Gamma,\lambda_c)
:=
\left(\lambda_c^{-1}+\frac{1-e^{-2\Gamma q_j}}{q_j}\right)^{-1}.
\label{eq:finite_transport_precision}
\end{equation}
Thus $c_j$ is not $\tilde q_j+\lambda_c$, which would correspond to adding a second precision penalty after evaluating the endpoint density.
The two limits recover familiar special cases: $\lambda_c\to\infty$ gives hard endpoint matching ($c_j=\tilde q_j(\Gamma)$), and the joint $\Gamma,\lambda_c\to\infty$ limit gives equilibrium ($c_j=q_j$).
The corresponding soft-endpoint log-likelihood is
\begin{equation}
\log p(\mathcal D\mid\varphi,\Gamma,\lambda_c)
=
-\frac{1}{2}\sum_{j=1}^{k} c_j(\Gamma,\lambda_c)\,S_j(\Gamma)
+\frac{D}{2}\sum_{j=1}^{k}\log c_j(\Gamma,\lambda_c)
+\mathrm{const}.
\label{eq:soft_endpoint_ll}
\end{equation}
The nodewise control-energy score uses the quadratic data-fit part of \eqref{eq:soft_endpoint_ll}; the log-determinant term is global and does not decompose over nodes.
The corresponding nodewise contribution to the control energy is
\begin{equation}
C_i(\Gamma,\lambda_c)
:=
\frac{1}{2}\Big\|\sum_{j=1}^k \sqrt{c_j(\Gamma,\lambda_c)}\,[v_j]_i\,
\bigl(v_j^\top \Delta_\Gamma\bigr)\Big\|^2 ,
\label{eq:finite_transport_score}
\end{equation}
The nodewise scores sum to the projected global quadratic energy, $\sum_i C_i=\frac{1}{2}\sum_{j=1}^k c_jS_j$.
Because $Q_\rho$ couples nodes, $C_i$ is a nodewise contribution to the global control energy, not an independent nodewise energy.

In some regimes the absolute scale of $C_i$ is confounded by overall residual magnitude, so a high-magnitude normal node can look anomalous under raw $C_i$ even when its spectral profile is consistent with the template.
We address this with a scale-normalized control ratio,
\begin{equation}
CR_i(\Gamma,\lambda_c)
:=
\frac{C_i(\Gamma,\lambda_c)}{\|[\Delta_\Gamma]_{i,:}\|^2+\epsilon},
\label{eq:finite_transport_ratio}
\end{equation}
where $\epsilon>0$ is a small numerical-stability constant (set to $10^{-8}$ in practice). This isolates graph-spectral roughness from raw residual scale.
The equilibrium scores are recovered as the joint $\Gamma,\lambda_c\to\infty$ limit of the finite-horizon scores:
\begin{equation}
J^\star_i := \lim_{\Gamma,\lambda_c\to\infty} C_i(\Gamma,\lambda_c)
= \frac{1}{2}\Big\|\sum_{j=1}^k \sqrt{q_j}\,[v_j]_i\,(v_j^\top\Delta)\Big\|^2,
\qquad
R_i := \lim_{\Gamma,\lambda_c\to\infty} CR_i(\Gamma,\lambda_c).
\label{eq:anomaly_score}
\end{equation}
When $k=n$, $J^\star_i=\frac{1}{2}\|[Q_\rho^{1/2}\Delta]_{i,:}\|^2$ and the total score is the graph-Gaussian Mahalanobis energy.
When $k<n$, the control-energy numerator operates on the retained spectral subspace while the ratio denominator uses the full nodewise residual norm: the latter is intentional, because anomalies may place substantial residual energy in truncated rough modes, so the ratio asks how much of a node's total residual is explained by the retained spectral geometry.

Together, the family of $C_{\Gamma,\lambda_c}$ and $CR_{\Gamma,\lambda_c}$ at varying $(\Gamma,\lambda_c)$, with their $\Gamma,\lambda_c\to\infty$ limits $J^\star$ and $R$, constitutes the \emph{finite-horizon score bank} (the bank).
By Proposition~\ref{prop:equilibrium_opt} the soft-endpoint likelihood collapses to its equilibrium limit (Step~4 of the proof handles the endpoint-residual variant), so $\Gamma$ and $\lambda_c$ are predeclared resolutions rather than fitted parameters; their usefulness is judged downstream by unlabeled tail and stability diagnostics (\S\ref{subsec:prior_selection}).

\subsection{Empirical Bayes Prior Selection}
\label{subsec:prior_selection}

For each candidate template bandwidth $\gamma$, empirical Bayes selects $(\rho,\kappa)$ from the exact equilibrium residual-space likelihood, and the fitted precision is reused by every finite-horizon score in \eqref{eq:finite_transport_score}--\eqref{eq:finite_transport_ratio}. EB estimates the graph-Gaussian geometry, while the finite-horizon bank probes how anomaly evidence appears along the GOU relaxation path. Because $\Delta=T_{\gamma,\nu}X$ is $\gamma$-dependent, the outer $\gamma$ step is \emph{residual-space profile selection}: it optimizes the Gaussian fit in the transformed residual field and omits the raw-feature Jacobian $D\sum_{j=1}^k\log|1-(\gamma^2+\lambda_j)^{-\nu}|$. The selected $\gamma^\star$ is used by the finite-horizon families, whose $p$-values are computed under the fitted law of that residual; it is also the bandwidth that removes the most feature energy (Table~\ref{tab:gamma_profile}). The equilibrium anchor instead uses the grid value $\gamma_a$ nearest to the graph-statistic center $\gamma_0$, with its own EB fit (Appendix~\ref{app:regeneration}).

\begin{restatable}{theorem}{marglikelihood}\label{thm:marginal_likelihood}
At equilibrium ($\Gamma\to\infty$), assuming $q_j>0$ on every retained mode, the residual-space log-likelihood projected onto the $k$ retained modes (full likelihood when $k=n$) reduces, for fixed $\gamma$ and up to an additive $(\rho,\kappa)$-independent constant, to
\begin{equation}
\log p(\mathcal{D}\mid\varphi)
=
-\underbrace{\frac{1}{2}\sum_{j=1}^k q_j\,S_j}_{\text{data fit (control energy)}}
\;+\;\underbrace{\frac{D}{2}\sum_{j=1}^k\log q_j}_{\text{Gaussian normalizer}}
\label{eq:marginal_likelihood}
\end{equation}
where $S_j=\|\Delta^\top v_j\|^2$ is the spectral energy on eigenmode $j$.
\end{restatable}

\noindent The first term rewards high precision on low-energy modes; the second discourages making the covariance uniformly too diffuse. When $k=n$, the data-fit term equals the total equilibrium control energy $\sum_i J_i^\star$.

\begin{algorithm}[t]
\caption{Empirical Bayes fit and finite-horizon score selection}
\label{alg:prior_opt}
\begin{algorithmic}[1]
\REQUIRE Graph $(A, X)$
\ENSURE Optimized prior $\varphi^\star = (\rho^\star, \kappa^\star, \gamma^\star)$; final anomaly score
\STATE Compute/cache Laplacian eigenpairs $\{(\lambda_j,v_j)\}$ for each predeclared graph operator (input graph or deterministic affinity-refined operator)
\STATE Set candidate template, $\gamma$, PCA, horizon $\Gamma$, and endpoint-tolerance $\lambda_c$ grids from graph statistics (Appendix~\ref{app:selector_details})
\FOR{each candidate $\gamma$ in range}
    \STATE Compute template $\mathbf{m}(\gamma)$ (Eq.~\ref{eq:template}), residuals $\delta_i$, spectral energies $S_j$
    \FOR{each candidate $\kappa \in [0, \kappa_{\max}]$ \textit{(profile $\rho^\star(\kappa)$)}}
        \STATE $\rho^\star(\kappa) \leftarrow$ maximize over $\rho\in[0,1]$ by clamped Newton updates with endpoint check
    \ENDFOR
    \STATE $\kappa^\star(\gamma) \leftarrow \arg\max_{\kappa} \log p(\mathcal{D}\mid\rho^\star(\kappa),\kappa,\gamma)$
    \STATE $\rho^\star(\gamma)\leftarrow \rho^\star(\kappa^\star(\gamma))$; store $\mathcal L^\star(\gamma)=\log p(\mathcal D\mid\rho^\star(\gamma),\kappa^\star(\gamma),\gamma)$
\ENDFOR
\STATE $\gamma^\star \leftarrow \arg\max_\gamma \mathcal L^\star(\gamma)$; set $\rho^\star\leftarrow\rho^\star(\gamma^\star)$ and $\kappa^\star\leftarrow\kappa^\star(\gamma^\star)$
\STATE Construct finite-horizon scores $C_{\Gamma,\lambda_c}$, $CR_{\Gamma,\lambda_c}$ at $\gamma^\star$; anchor $J^\star,R$ at $\gamma_a$ with $(\rho^\star(\gamma_a),\kappa^\star(\gamma_a))$
\STATE Select score family from graph statistics (Appendix~\ref{app:selector_details})
\STATE Select the final score or rank-fused profile by fitted-null deviation and neighboring-score stability
\STATE \textbf{return} anomaly scores
\end{algorithmic}
\end{algorithm}

\paragraph{Four score families.}
The bank is used through four families:
\begin{itemize}[nosep,leftmargin=*]
\item \emph{Equilibrium anchor} ($J^\star$ or $R$): the stationary graph-Mahalanobis energy or its scale-normalized ratio. It is the default, used when the finite-horizon diagnostics are not stable enough to justify leaving the stationary score, and on sparse, heavily truncated graphs.
\item \emph{Control energy} ($C_{\Gamma,\lambda_c}$): the absolute effort to reach a node at a finite horizon. It is used when residual magnitude is itself informative: on dense graphs with homophilic low-dimensional features, or with high-dimensional features of near-zero homophily (then on the affinity template).
\item \emph{Control ratio} ($CR_{\Gamma,\lambda_c}$): the share of a node's residual that the fitted graph geometry makes expensive. It is used when residual magnitude would otherwise dominate: sparse graphs with high-dimensional, weakly homophilic features.
\item \emph{Profile aggregation}: per-node $p$-values of the bank members under the shared plug-in null, combined across horizons and tolerances. It is used when evidence is spread over horizons: very dense graphs, or graphs whose profile diagnostics pass.
\end{itemize}

\paragraph{Label-free selector.}
After the EB fit, a fixed rule chooses one family using only label-free statistics: cosine feature homophily $h$, edge density $d_e$, feature dimension $D_x$, and the predeclared truncation $k$.
Within the chosen family, scores in the admissible pool $\mathcal B$ are ranked by percentile-normalized quality plus neighboring-rank stability,
$A(r)=\operatorname{pct}_{\mathcal B}(Q(r))+\operatorname{pct}_{\mathcal B}(S(r))$, where $Q(r)$ is a family-specific fitted-null diagnostic and $S(r)=\operatorname{median}_{r'\in N(r)}\operatorname{corr}(\operatorname{rank}(s_r),\operatorname{rank}(s_{r'}))$ is local rank-stability over predeclared neighbors $N(r)$ of the node-score vector $s_r\in\mathbb{R}^n$; ranks average ties, so the selection does not depend on the order of the nodes.
The selector is label-free model selection: the graph statistics decide which family's null model fits the regime, and the within-family diagnostics are goodness-of-fit statistics against that family's plug-in null. It is sound in this sense, but it is not derived from a single principle: its cutoffs are implementation defaults, identical on all datasets, whose provenance and tested sensitivity are given in Appendix~\ref{app:selector_details} with the ordered rules.

The likelihood is concave in $\rho$, while $\kappa$ and $\gamma$ are searched on bounded grids (Remark~\ref{rem:landscape}); Remark~\ref{rem:outlier} discusses the sensitivity of the fit to anomalies.

\FloatBarrier
\section{Experiments}
\label{sec:experiments}

\begin{table}[t]
  \centering
  \caption{%
    AUROC (\%) on 11 GAD benchmarks. \textbf{Bold}: best; \underline{underline}: second; \oom{} = out of memory. Baselines use their native anomaly-score direction (Appendix Table~\ref{tab:score_orientation_audit}). Cell sources: $^{*}$ published in the TAM benchmark \cite{qiao2023truncated}; $^{\dagger}$ published by the baseline's authors (out of memory on our hardware); unmarked baseline cells are our runs with published hyperparameters. For truncated spectra (YelpChi, Elliptic, Elliptic++, DGraph, T-Finance) EB-GAD is the mean $\pm$ standard deviation over five node orders (Appendix~\ref{app:regeneration}).
  }
  \label{tab:main_auroc}
  \scriptsize
  \setlength{\tabcolsep}{2.4pt}
  \renewcommand{\arraystretch}{0.9}
  \resizebox{\textwidth}{!}{%
  \begin{tabular}{l cccccccccc}
    \toprule
    Dataset
      & LOF & DIF & ANOM. & DOMINANT & AnomDAE & CONAD & CoLA & DiffGAD & TAM
      & \textbf{EB-GAD} \\
    \midrule
    Weibo
      & $57.9$ & $38.2\,{\scriptstyle\pm\,1.8}$ & $\underline{94.5}\,{\scriptstyle\pm\,0.0}$ & $85.0\,{\scriptstyle\pm\,14.6}$ & $91.5\,{\scriptstyle\pm\,1.2}$ & $85.4\,{\scriptstyle\pm\,14.3}$ & $28.1\,{\scriptstyle\pm\,0.0}$ & $93.4\,{\scriptstyle\pm\,0.3}$ & $70.9\,{\scriptstyle\pm\,0.0}$ & $\mathbf{95.1}$ \\
    Reddit
      & $57.2$ & $53.5\,{\scriptstyle\pm\,1.4}$ & $54.9\,{\scriptstyle\pm\,5.6}$ & $56.0\,{\scriptstyle\pm\,0.2}$ & $55.7\,{\scriptstyle\pm\,0.4}$ & $56.1\,{\scriptstyle\pm\,0.1}$ & $\underline{60.3}\,{\scriptstyle\pm\,0.7}$$^{*}$ & $56.3\,{\scriptstyle\pm\,0.1}$ & $60.2\,{\scriptstyle\pm\,0.4}$$^{*}$ & $\mathbf{60.6}$ \\
    Amazon
      & $47.4$ & $\underline{77.1}\,{\scriptstyle\pm\,2.5}$ & $44.6\,{\scriptstyle\pm\,0.3}$$^{*}$ & $60.0\,{\scriptstyle\pm\,0.4}$$^{*}$ & $64.6\,{\scriptstyle\pm\,7.1}$ & $55.5\,{\scriptstyle\pm\,1.6}$ & $59.0\,{\scriptstyle\pm\,0.8}$$^{*}$ & $55.9\,{\scriptstyle\pm\,0.1}$ & $70.6\,{\scriptstyle\pm\,1.0}$$^{*}$ & $\mathbf{78.0}$ \\
    YelpChi
      & $\underline{57.2}$ & $56.0\,{\scriptstyle\pm\,0.9}$ & $49.6\,{\scriptstyle\pm\,0.3}$$^{*}$ & $41.3\,{\scriptstyle\pm\,1.0}$$^{*}$ & $56.6\,{\scriptstyle\pm\,2.6}$ & $28.7\,{\scriptstyle\pm\,0.2}$ & $46.4\,{\scriptstyle\pm\,0.1}$$^{*}$ & $29.1\,{\scriptstyle\pm\,0.0}$ & $56.4\,{\scriptstyle\pm\,0.7}$$^{*}$ & $\mathbf{71.7}\,{\scriptstyle\pm\,0.6}$ \\
    BlogCatalog
      & $68.6$ & $57.4\,{\scriptstyle\pm\,2.1}$ & $56.5\,{\scriptstyle\pm\,2.5}$$^{*}$ & $75.9\,{\scriptstyle\pm\,1.0}$$^{*}$ & $74.7\,{\scriptstyle\pm\,0.0}$ & $68.7\,{\scriptstyle\pm\,0.0}$ & $77.5\,{\scriptstyle\pm\,0.9}$$^{*}$ & $76.6\,{\scriptstyle\pm\,0.0}$ & $\mathbf{82.5}\,{\scriptstyle\pm\,0.3}$$^{*}$ & $\underline{78.6}$ \\
    Facebook
      & $47.7$ & $37.8\,{\scriptstyle\pm\,2.8}$ & $90.2\,{\scriptstyle\pm\,0.5}$$^{*}$ & $56.8\,{\scriptstyle\pm\,0.2}$$^{*}$ & $54.7\,{\scriptstyle\pm\,0.6}$ & $14.2\,{\scriptstyle\pm\,1.5}$ & $84.3\,{\scriptstyle\pm\,1.1}$$^{*}$ & $28.0\,{\scriptstyle\pm\,4.7}$ & $\mathbf{91.4}\,{\scriptstyle\pm\,0.8}$$^{*}$ & $\mathbf{91.4}$ \\
    ACM
      & $72.0$ & $49.2\,{\scriptstyle\pm\,1.2}$ & $68.6\,{\scriptstyle\pm\,6.3}$$^{*}$ & $85.7\,{\scriptstyle\pm\,2.0}$$^{*}$ & $74.8\,{\scriptstyle\pm\,0.0}$ & $80.0\,{\scriptstyle\pm\,0.1}$ & $82.3\,{\scriptstyle\pm\,0.1}$$^{*}$ & $81.8\,{\scriptstyle\pm\,1.9}$ & $\mathbf{88.8}\,{\scriptstyle\pm\,2.4}$$^{*}$ & $\underline{86.0}$ \\
    Elliptic
      & $45.1$ & $21.8\,{\scriptstyle\pm\,0.5}$ & \oom & $11.2\,{\scriptstyle\pm\,3.5}$ & $14.2\,{\scriptstyle\pm\,1.2}$ & $11.3\,{\scriptstyle\pm\,0.2}$ & $\underline{65.3}\,{\scriptstyle\pm\,1.3}$ & $26.6\,{\scriptstyle\pm\,1.0}$ & $36.0\,{\scriptstyle\pm\,0.2}$ & $\mathbf{74.5}\,{\scriptstyle\pm\,0.3}$ \\
    Elliptic++
      & $45.9$ & $22.5\,{\scriptstyle\pm\,0.7}$ & \oom & $39.7\,{\scriptstyle\pm\,0.0}$ & $23.7\,{\scriptstyle\pm\,3.9}$ & $39.3\,{\scriptstyle\pm\,0.9}$ & $\underline{62.9}\,{\scriptstyle\pm\,6.7}$ & $18.4\,{\scriptstyle\pm\,2.3}$ & $34.4\,{\scriptstyle\pm\,0.2}$ & $\mathbf{72.5}\,{\scriptstyle\pm\,0.0}$ \\
    DGraph
      & $40.4$ & $\underline{64.2}\,{\scriptstyle\pm\,1.8}$ & \oom & \oom & \oom & $34.7\,{\scriptstyle\pm\,1.2}$$^{\dagger}$ & $54.8\,{\scriptstyle\pm\,0.7}$ & $52.4\,{\scriptstyle\pm\,0.0}$$^{\dagger}$ & $35.7\,{\scriptstyle\pm\,0.8}$ & $\mathbf{66.8}\,{\scriptstyle\pm\,0.1}$ \\
    T-Finance
      & $34.5$ & $35.5\,{\scriptstyle\pm\,1.5}$ & $55.4\,{\scriptstyle\pm\,0.0}$ & $53.8\,{\scriptstyle\pm\,0.0}$ & $32.1\,{\scriptstyle\pm\,0.0}$ & $\underline{63.1}\,{\scriptstyle\pm\,0.1}$ & $48.3\,{\scriptstyle\pm\,0.0}$ & $53.6\,{\scriptstyle\pm\,0.1}$ & $61.8\,{\scriptstyle\pm\,0.0}$$^{\dagger}$ & $\mathbf{84.8}\,{\scriptstyle\pm\,0.0}$ \\
    \bottomrule
  \end{tabular}
  }
\end{table}

\textbf{Setup.}
We evaluate on 11 benchmarks (7 standard GAD graphs and 4 financial networks up to 3.7M nodes; Table~\ref{tab:main_auroc}) against LOF, DIF, and seven graph-based baselines \cite{lof_2000,DIF,chen2025pyod,anomalous,ding2019deep,fan2020anomalydae,conad,Liu_2022_cola,pygod,qiao2023truncated,li2024diffgad}, including TAM, a strong local-affinity baseline.
On Amazon and YelpChi we follow the TAM/DiffGAD evaluation setup \cite{qiao2023truncated,li2024diffgad}; all AUROCs are computed on the same node set across methods.
EB-GAD is run by one fixed procedure on every dataset: the grids, the ordered rules and their cutoffs are those of Appendix~\ref{app:selector_details}, and no label enters any step before evaluation.
Every EB-GAD number was regenerated for this version and passes a \emph{relabeling gate}: a score must not change when the nodes are renumbered. Appendix~\ref{app:regeneration} gives the protocol and lists what changed from the submitted version.

\textbf{Main results.}
EB-GAD has the best or tied-best AUROC on 9 of 11 datasets.
It leads on the four financial networks (T-Finance, Elliptic, Elliptic++, DGraph) and on YelpChi by $2.6$ to $21.7$ points over the best baseline, leads on Amazon by $0.9$ (within the standard deviation of DIF, $77.1\pm2.5$) and on Weibo by $0.6$, and ties TAM on Facebook ($91.4$).
On Reddit its $60.6$ is level with CoLA ($60.3\pm0.7$) and TAM ($60.2$); it is second to TAM on BlogCatalog and ACM.
The Weibo and Reddit entries depend on the anchor bandwidth $\gamma_a$ (\S\ref{subsec:prior_selection}), a rule we adopted after evaluating the anchor at the likelihood-selected bandwidth, which gives $92.5$ and $48.5$ and a count of 7 of 11 (Appendix~\ref{app:regeneration}).
AUPRC (Appendix Table~\ref{tab:auprc}) confirms the five fraud graphs, with the largest margin on T-Finance ($53.0$ against $6.6$ for the best baseline), and is less favorable elsewhere: EB-GAD is ahead on Weibo by $0.3$, level with the best on Reddit ($4.5$), second on Amazon, third on BlogCatalog and ACM, and fourth on Facebook, where it trails TAM, CoLA and ANOMALOUS despite the tied AUROC.
BlogCatalog ($h=0.009$) and ACM are routed to the feature-affinity template; high-pass and band-pass variants of it do not improve BlogCatalog (Appendix~\ref{app:templates}).
Appendix Tables~\ref{tab:ebgad_hyperparams}--\ref{tab:selector_graph_stats} give the configuration, rule and statistics behind each cell.

\textbf{What the finite horizon contributes.}
Table~\ref{tab:bank_ablation} compares each selected score with the equilibrium score of the same configuration, so the fitted model and everything else are held fixed.
Six datasets are routed to an equilibrium anchor.
Of the five routed to a finite-horizon family, three gain ($+3.6$ on Facebook and $+4.7$ on T-Finance by profile aggregation, $+4.7$ on ACM by the control ratio) and the two control-energy fusions are at parity ($+0.1$, $+0.2$).
A narrower comparison agrees: a single finite-horizon score, with no tolerance grid, fusion or aggregation, is within $0.3$ points of its equilibrium limit or below it (Appendix~\ref{app:horizon_studies}).
The gain does not come from one fortunate horizon but from aggregating and normalizing a family of scores that share one fitted prior and one plug-in null.

\begin{table}[t]
  \centering
  \scriptsize
  \setlength{\tabcolsep}{2.2pt}
  \caption{Contribution of the finite horizon: AUROC (\%) of the selected score against the equilibrium score ($J^\star$ or $R$, by NullKS) of the \emph{same} configuration (operator, template, bandwidth, PCA, truncation, EB fit). \emph{Family}: the family the selector routes to; for an anchor the two rows coincide.}
  \label{tab:bank_ablation}
  \resizebox{\textwidth}{!}{%
  \begin{tabular}{l ccccccccccc}
    \toprule
    & Weibo & Reddit & Amazon & YelpChi & BlogCat. & Facebook & ACM & Elliptic & Ell.++ & DGraph & T-Fin. \\
    \midrule
    Family & anchor & anchor & energy & anchor & energy & profile & ratio & anchor & anchor & anchor & profile \\
    Equilibrium & 95.1 & 60.6 & 77.9 & 71.7 & 78.4 & 87.8 & 81.3 & 74.5 & 72.5 & 66.8 & 80.1 \\
    Selected & 95.1 & 60.6 & \textbf{78.0} & 71.7 & \textbf{78.6} & \textbf{91.4} & \textbf{86.0} & 74.5 & 72.5 & 66.8 & \textbf{84.8} \\
    $\Delta$ & $0.0$ & $0.0$ & $+0.1$ & $0.0$ & $+0.2$ & $+3.6$ & $+4.7$ & $0.0$ & $0.0$ & $0.0$ & $+4.7$ \\
    \bottomrule
  \end{tabular}
  }
\end{table}

\textbf{Isolation, selector and robustness studies.}
The appendix reports the studies behind these claims, all rerun for this version.
\emph{Isolation} (Appendix~\ref{app:isolation}): with residuals and selection fixed and only the spectral precision changed, the GOU bank leads heat-kernel, polynomial and identity precisions by $10.2$ to $12.6$ points on Elliptic, Elliptic++ and T-Finance; elsewhere no family is consistently better, and on Facebook a heat-kernel bank is ahead.
\emph{Selector} (Appendix~\ref{app:selector_studies}): perturbing each of its 21 cutoffs by $\pm25\%$ leaves the selection unchanged in $99.4\%$ of the evaluations ($97.4\%$ at $\pm50\%$); quality or stability alone fail on Facebook ($17.6$ and $25.5$ AUROC) where their combination gives $91.4$; a supervised combination of the same scores would gain less than one point on four datasets and $8$ to $17$ points on four.
\emph{Robustness} (Appendix~\ref{app:robustness}): $24\%$ contamination costs $1.1$ points on Weibo, heavy-tailed residuals at most $0.7$, a multimodal normal population $9$ to $19$, and binarized features defeat the Gaussian residual model.

\textbf{Encoders and scalability.}\label{subsec:scalability}
Replacing PCA by a trained MLP or GraphSAGE encoder \cite{hamilton2017graphsage} breaks the linear-Gaussian residual model that EB exploits in closed form: the training-free pipeline is ahead on 8 of 11 datasets, and a trained encoder on Weibo, Facebook and, by $0.1$, T-Finance (Appendix~\ref{app:encoder}).
On 16 CPU threads the complete run behind a cell of Table~\ref{tab:main_auroc} takes from $8$ seconds (Facebook) to $33$ minutes (DGraph, of which $30$ are the eigendecomposition); the two trained graph baselines that fit in memory on DGraph need $65$ and $69$ minutes on a GPU (Appendix Table~\ref{tab:runtime}).

\section{Scope, Limitations and Conclusion}
\label{sec:conclusion}

EB-GAD is a training-free empirical-Bayes model for node-level graph anomaly detection: a graph-aware GOU defines a bank of finite-horizon control-energy scores, the residual-space likelihood selects the graph precision without labels, and a fixed selector chooses one of four score families. The equilibrium Mahalanobis score is one limiting case; aggregating or normalizing the finite-horizon family under its shared null adds $3.6$ to $4.7$ AUROC points on three datasets and is at parity elsewhere. Run without labels, EB-GAD has the best or tied-best AUROC on 9 of 11 benchmarks.

\textbf{Scope and limitations.}
(i) \emph{Selection.} The model does not identify the template bandwidth: the likelihood sets it for the finite-horizon families and graph statistics set it for the anchor, a rule adopted after the first had been evaluated and without which Weibo and Reddit are $92.5$ and $48.5$. The selector is label-free and stable under perturbation of its cutoffs, but it is not derived from one principle: its cutoffs are implementation defaults developed on the benchmarks of this paper, and we do not claim that they transfer to other graph families.
(ii) \emph{Gaussian model.} In a controlled test heavy-tailed residuals cost at most $0.7$ points and a two-component normal population $9$ to $19$; binarized features fall outside the Gaussian residual model (AUROC below $61$) and need a categorical likelihood.
(iii) \emph{Contamination.} With up to $24\%$ contamination the fitted prior moves little and detection of the original anomalies loses $1.1$ points on Weibo; large contiguous anomalous regions, which can become their own context, were not studied.
(iv) \emph{Graph operators.} BlogCatalog and ACM trail TAM, and high-pass or band-pass templates do not close the gap; heterophily is handled only through feature-affinity reweighting.
(v) \emph{Truncated spectra.} With many small components the truncated eigenbasis is not unique; on YelpChi the AUROC ranges over two points across node orders.
(vi) \emph{Fairness.} Scores on user graphs can encode degree and activity biases and should be audited across user segments before deployment.

\label{endofmain}
\par\typeout{ENDOFMAIN page=\thepage\space used=\the\pagetotal\space of \the\pagegoal}

\begin{ack}
We thank Block for supporting this work; part of it was carried out while Fred Xu was an intern at Block. This work was partially supported by the National Science Foundation under grant 2531008.
\end{ack}

\bibliographystyle{plain}
\bibliography{paper}

\newpage
\appendix
\section{Derivation of Technical Results}
\label{sec:proofs}

We provide proofs of the propositions and theorem stated in Section~\ref{subsec:graph_prior}--\ref{subsec:prior_selection}.

\subsection{Modeling Details}
\label{app:modeling_details}

\paragraph{Notation: vec convention.}
We use the row-major vectorization throughout, so $\bar Q_\rho=Q_\rho\otimes I_D$ acts as $\bar Q_\rho\,\operatorname{vec}(X)=\operatorname{vec}(Q_\rho X)$ for any $X\in\mathbb{R}^{n\times D}$. Equivalently, $Q_\rho$ couples nodes while $I_D$ leaves the feature direction untouched.

\paragraph{Residual template and modeled subspace.}
The template operator $S_{\gamma,\nu}$ in \eqref{eq:template} is low-pass but is not normalized to have unit gain on the Laplacian nullspace; consequently $\gamma$ controls both bandwidth and low-frequency scale.
This scale is part of the profiled template family, and the same residual operator is used consistently in likelihood evaluation and scoring.
For fixed $\gamma$, each feature column of the residual field $\Delta=T_{\gamma,\nu}X$ has the modeled full-spectrum law $\mathcal{N}(0,Q_\rho^{-1})$; equivalently, the vectorized residual has covariance $\bar Q_\rho^{-1}=Q_\rho^{-1}\otimes I_D$.
The Laplacian supplies graph-aware anisotropy and the $(1-\rho)I_n$ term supplies an isotropic background.
When $k<n$, the likelihood and the energy score $J^\star$ are computed on the projected residual $P_k\Delta$, while the ratio score $R$ normalizes this projected energy by the full nodewise residual norm.
Feature dimensions are z-score standardized before template construction and scoring.
Disconnected graphs may still have additional Laplacian null modes; we do not remove them explicitly, and instead require only $q_j>0$ on every retained mode so graph-uninformative directions are handled by the isotropic part of $Q_\rho$.

\begin{remark}[Optimization landscape]\label{rem:landscape}
Writing $a_j:=(\kappa^2+\lambda_j)^\nu-1$ and $q_j=\rho a_j+1$, the residual-space log-likelihood has gradient $\partial\log p/\partial\rho=\sum_j a_j(-S_j/2+D/(2q_j))$ and Hessian $\partial^2\log p/\partial\rho^2=-\sum_j D a_j^2/(2q_j^2)\le 0$, so it is concave in $\rho$ with at most one interior root; Algorithm~\ref{alg:prior_opt} locates it by clamped Newton steps and otherwise compares the boundary values $\rho\in\{0,1\}$. The profile likelihood is non-concave in $\kappa$ and the template-induced objective is non-convex in $\gamma$, so both are searched on bounded grids.
\end{remark}

\begin{remark}[Outlier sensitivity]\label{rem:outlier}
The squared-residual data-fit $\tfrac{1}{2}\sum_j q_j S_j$ is in principle sensitive to extreme anomalies, but three properties mitigate this here: anomaly fractions are 1--10\% so the bulk of the likelihood is set by normal nodes; rough-mode energy that anomalies inflate is dropped by the truncated eigenspace on large graphs; and the $\tfrac{D}{2}\sum_j\log q_j$ normalizer prevents a small high-residual group from collapsing $\rho$ to zero. Appendix~\ref{app:robustness} measures the effect of added contamination; robust pseudolikelihoods are a natural extension.
\end{remark}

\paragraph{Connection to stochastic optimal control.}
\label{rem:soc_connection}
We use $\lambda_c$ as an inverse endpoint-tolerance variance: the latent GOU endpoint has covariance $\Sigma_\Gamma$, the observation tolerance has covariance $\lambda_c^{-1}I$, and marginalizing the latent endpoint gives the quadratic form induced by $(\lambda_c^{-1}I+\Sigma_\Gamma)^{-1}$.
Under hard endpoint matching, $\lambda_c\to\infty$, this reduces to $\frac{1}{2}(x_T-\mu_\Gamma)^\top\Sigma_\Gamma^{-1}(x_T-\mu_\Gamma)$.
Equation~\eqref{eq:finite_transport_score} is the per-node \emph{contribution} to this total control energy, not an independent nodewise control energy because $Q_\rho$ couples nodes. The energy is that of steering one process to one observed endpoint; no coupling between distributions is optimized.
For general nonlinear bridge models, computing the analogous endpoint control energy usually requires a learned score, controller, or reverse transition model \cite{zhou2024ddbm,yue2023goub,zhu2025unidb}; here the Gaussian transition kernel gives it directly, making such training unnecessary.

\reftransition*

\begin{proof}
Fix $0\le s<t\le T$. Define the centered process $Y_t:=X_t-\mathbf m$. Then \eqref{eq:ref_graph_ou} becomes
\begin{equation}
dY_t = -\alpha(t)\bar Q_\rho\,Y_t\,dt + \sqrt{2\alpha(t)}\,dW_t,
\qquad Y_s = x_s-\mathbf m.
\label{eq:centered_SDE}
\end{equation}
This is a linear SDE with time-dependent scalar coefficient $\alpha(t)$ multiplying the constant symmetric positive-definite matrix $\bar Q_\rho$ on the modeled subspace.

\paragraph{Step 1: Variation of constants.}
Let $\Phi_{s:t}:=\exp(-\Gamma_{s:t}\bar Q_\rho)$ with $\Gamma_{s:t}:=\int_s^t\alpha(r)\,dr$ be the fundamental solution, satisfying $\partial_t \Phi_{s:t} = -\alpha(t)\bar Q_\rho\,\Phi_{s:t}$.
Applying It\^o's product rule to $\Phi_{s:t}^{-1}Y_t$:
\[
d(\Phi_{s:t}^{-1}Y_t) = \alpha(t)\bar Q_\rho\,\Phi_{s:t}^{-1}Y_t\,dt + \Phi_{s:t}^{-1}\bigl(-\alpha(t)\bar Q_\rho\,Y_t\,dt + \sqrt{2\alpha(t)}\,dW_t\bigr) = \Phi_{s:t}^{-1}\sqrt{2\alpha(t)}\,dW_t.
\]
Integrating and left-multiplying by $\Phi_{s:t}$:
\[
Y_t = \Phi_{s:t}Y_s + \int_s^t \Phi_{\tau:t}\sqrt{2\alpha(\tau)}\,dW_\tau.
\]

\paragraph{Step 2: Conditional mean and covariance.}
Taking expectation: $\mathbb{E}[Y_t\mid Y_s]=\Phi_{s:t}Y_s$, giving $\mu_{t|s}(x_s) = \mathbf m + \Phi_{s:t}(x_s-\mathbf m)$.
By the It\^o isometry (using $\bar Q_\rho$ symmetric, so $\Phi_{\tau:t}^\top = \Phi_{\tau:t}$):
\[
\Sigma_{t|s} = \int_s^t 2\alpha(\tau)\,\Phi_{\tau:t}^2\,d\tau
= \int_s^t 2\alpha(\tau)\,\exp(-2\Gamma_{\tau:t}\bar Q_\rho)\,d\tau.
\]

\paragraph{Step 3: Closed form via substitution.}
Substituting $u=\Gamma_{\tau:t}=\int_\tau^t\alpha(r)\,dr$ gives $du = -\alpha(\tau)\,d\tau$. The limits are $u=\Gamma_{s:t}$ when $\tau=s$ and $u=0$ when $\tau=t$, so:
\[
\Sigma_{t|s}
= \int_0^{\Gamma_{s:t}} 2\,e^{-2u\bar Q_\rho}\,du
= \bar Q_\rho^{-1}\bigl(I_{nD}-e^{-2\Gamma_{s:t}\bar Q_\rho}\bigr),
\]
where the last equality uses $\int_0^a 2e^{-2uA}\,du = A^{-1}(I-e^{-2aA})$ for any SPD matrix $A$.
\end{proof}

\paragraph{Stationary distribution.}
When $\bar Q_\rho$ is SPD (or after restricting to the modeled subspace), the GOU process \eqref{eq:ref_graph_ou} has stationary distribution $\mathcal{N}(\mathbf m, \bar Q_\rho^{-1})$.
This follows from Proposition~\ref{prop:ref_transition_score}: if $X_0 \sim \mathcal{N}(\mathbf m, \bar Q_\rho^{-1})$, then in each retained mode
\[
\mathrm{Var}(X_T)_j = e^{-2\Gamma q_j}/q_j + (1-e^{-2\Gamma q_j})/q_j = 1/q_j,
\]
recovering the initial covariance.
The mean is also preserved because $\mathbb{E}[X_T]=\mathbf m+e^{-\Gamma\bar Q_\rho}(\mathbb{E}[X_0]-\mathbf m)=\mathbf m$ whenever $\mathbb{E}[X_0]=\mathbf m$.
Because every retained mode has $q_j>0$, the mean term $e^{-\Gamma q_j}(x_0-\mathbf m)_j$ and the initial-covariance contribution $e^{-2\Gamma q_j}\operatorname{Var}(X_0)_j$ vanish as $\Gamma\to\infty$ for any initial law with finite second moment, so this fixed point is also the limiting stationary distribution.
The equilibrium covariance $\bar Q_\rho^{-1}$ is the $\Gamma\to\infty$ limit of the finite-horizon covariance $\Sigma_\Gamma$ used in the generative model \eqref{eq:finite_time_predictive}.

\marglikelihood*

\begin{proof}
Fix $\gamma$ and let $\Delta = T_{\gamma,\nu}X$ with $T_{\gamma,\nu}=I_n-S_{\gamma,\nu}$ as in Section~\ref{subsec:graph_prior}.
The Kronecker structure $\bar Q_\rho = Q_\rho \otimes I_D$ means the $D$ feature dimensions are conditionally independent given the graph.
Let $V_k=[v_1,\dots,v_k]\in\mathbb{R}^{n\times k}$ and $P_k = V_kV_k^\top$ be the projection onto the $k$ computed eigenmodes.
For each feature dimension $\ell$, define the retained spectral coordinates
\[
z_\ell := V_k^\top \Delta_{:,\ell}\in\mathbb{R}^k,\qquad \Delta_{:,\ell}\in\mathbb{R}^n \text{ the $\ell$-th column of } \Delta\in\mathbb{R}^{n\times D}.
\]
Under the modeled residual law, these coordinates are Gaussian with diagonal covariance
\[
z_\ell \sim \mathcal{N}\!\bigl(0,\operatorname{diag}(1/q_1,\dots,1/q_k)\bigr),
\]
since $Q_\rho v_j = q_j v_j$ on the retained subspace.
Therefore the log-density factorizes across feature dimensions:
\[
\log p(P_k\Delta\mid\varphi) = \sum_{\ell=1}^D \log p(z_\ell\mid\varphi)
= -\tfrac{1}{2}\sum_{j=1}^k q_j S_j + \tfrac{D}{2}\sum_{j=1}^k\log q_j - \tfrac{kD}{2}\log 2\pi,
\]
where
\[
S_j = \sum_{\ell=1}^D (v_j^\top\Delta_{:,\ell})^2 = \|\Delta^\top v_j\|^2.
\]
When $k=n$, $P_n = I_n$ and the modeled residual object is $\Delta$ itself, yielding \eqref{eq:marginal_likelihood} up to the explicit additive constant $-\tfrac{nD}{2}\log 2\pi$.
When $k < n$, the projected likelihood is the marginal likelihood of the data restricted to the computed spectral subspace; the additive constant becomes $-\tfrac{kD}{2}\log 2\pi$, which is independent of $(\rho,\kappa)$ and therefore irrelevant for model selection.
This is consistent with the energy part of \eqref{eq:anomaly_score}, which uses the same $k$-mode projection; ratio scores then normalize that projected energy by the full nodewise residual norm as described in Section~\ref{subsec:anomaly_score}.
\end{proof}

\equilibriumopt*

\begin{proof}
\noindent\textbf{Step 1: Precision bound.}
Since $0 < 1-e^{-2\Gamma q_j} < 1$ for all finite $\Gamma > 0$ and $q_j > 0$:
\[
\tilde q_j(\Gamma) = \frac{q_j}{1-e^{-2\Gamma q_j}} > q_j.
\]
Equality holds only in the limit $\Gamma \to \infty$.

\noindent\textbf{Step 2: Convergence rate.}
Let $r_j(\Gamma) := \tilde q_j(\Gamma)/q_j = 1/(1-e^{-2\Gamma q_j})$.
For $\Gamma q_j \ge 1$, $e^{-2\Gamma q_j} \le e^{-2}$, so $r_j \le 1/(1-e^{-2}) \approx 1.16$.
More generally, $r_j - 1 = e^{-2\Gamma q_j}/(1-e^{-2\Gamma q_j}) \le 2e^{-2\Gamma q_j}$ for $\Gamma q_j \ge \frac{1}{2}\log 2$, establishing the exponential bound \eqref{eq:precision_convergence}.

\noindent\textbf{Step 3: Objective convergence for a fixed residual object.}
Let $Y$ be any fixed residual object with retained spectral energies $S_j(Y)$.
Write $\tilde q_j = q_j \cdot r_j$ and expand the log-likelihood difference:
\begin{align*}
\ell_\Gamma(Y;\varphi) - \ell_\infty(Y;\varphi)
&= \sum_j \Bigl[-\frac{q_j S_j(Y)}{2}(r_j - 1) + \frac{D}{2}\log r_j\Bigr].
\end{align*}
Since $r_j \to 1$ exponentially as $\Gamma \to \infty$ (Step~2), both $(r_j-1)$ and $\log r_j$ vanish, giving pointwise convergence.
Using $\log r_j\le r_j-1$ and the bound from Step~2, if $\Gamma q_{\min}\ge \frac{1}{2}\log 2$ then
\[
\bigl|\ell_\Gamma(Y;\varphi)-\ell_\infty(Y;\varphi)\bigr|
\le \frac{1}{2}\sum_{j=1}^k (q_jS_j(Y)+D)(r_j-1)
\le C(Y,\varphi)e^{-2\Gamma q_{\min}},
\]
with $C(Y,\varphi)=\sum_{j=1}^k(q_jS_j(Y)+D)$, matching the bound in Proposition~\ref{prop:equilibrium_opt}.
Moreover, the convergence is uniform over modes: for $\Gamma \ge c/q_{\min}$, $\max_j |r_j - 1| \le 2e^{-2c}$.

\noindent\textbf{Step 4: Endpoint-residual convergence.}
The proposition above fixes the residual object because this is the quantity optimized by the residual-space likelihood.
For the endpoint score, the residual also depends on $\Gamma$:
\[
\Delta_\Gamma=\Delta+e^{-\Gamma Q_\rho}M,
\qquad
v_j^\top\Delta_\Gamma=v_j^\top\Delta+e^{-\Gamma q_j}v_j^\top M,
\]
where $M$ is the template reshaped as an $n\times D$ matrix; this is matrix form, so $\bar Q_\rho=Q_\rho\otimes I_D$ acts trivially on the feature direction.
On the retained subspace,
\[
\|P_k(\Delta_\Gamma-\Delta)\|_F\le e^{-\Gamma q_{\min}}\|P_kM\|_F .
\]
Thus the endpoint residual converges to the equilibrium residual, and the full endpoint score converges to the equilibrium score with an $O(e^{-\Gamma q_{\min}})$ contribution from the mean and an $O(e^{-2\Gamma q_{\min}})$ contribution from the precision.

\noindent\textbf{Step 5: Practical indistinguishability.}
At the selected $(\rho^\star, \kappa^\star)$ and $\nu=1$, the minimum precision is bounded below by
\[
q_{\min}=(1-\rho^\star)+\rho^\star(\kappa^{\star\,2}+\lambda_{\min}^{\mathrm{ret}})
\ge (1-\rho^\star) + \rho^\star \kappa^{\star\,2},
\]
where $\lambda_{\min}^{\mathrm{ret}}$ is the smallest retained eigenvalue; equality holds when a zero Laplacian mode is retained.
This quantity is strictly positive whenever $\rho^\star<1$ or $\kappa^\star>0$.
Hence for any $c \ge \tfrac{1}{2}\log 2$ and any horizon satisfying $\Gamma \ge c/q_{\min}$,
\[
\max_j (r_j-1) \le 2e^{-2c}.
\]
For example, $c=3$ gives a uniform relative precision error below $0.5\%$ on every retained mode. Thus the finite-horizon and equilibrium objectives become practically indistinguishable once $\Gamma$ is a few multiples of $1/q_{\min}$, exactly as stated in Proposition~\ref{prop:equilibrium_opt}.
\end{proof}

\section{Experimental Details}

\subsection{Dataset Specifications}

Here we provide the details of each dataset. The detailed dataset statistics can be found in Table \ref{tab:dataset_stats}.

\begin{itemize}
    \item \textbf{Weibo} \cite{zhao_weibo}: Social network from the Tencent-Weibo platform. Nodes are users and edges represent follower/followee relationships. Labels mark suspicious or abusive accounts.
    \item \textbf{Reddit} \cite{wang_reddit}: Reddit interaction graph used in GAD benchmarks, with text-derived node attributes and labels associated with banned or anomalous users.
    \item \textbf{Amazon} \cite{dou2020caregnn,qiao2023truncated}: Review-derived user graph from Amazon. Nodes are reviewers; relations connect reviewers through shared review behavior such as reviewing the same product. Labels indicate suspicious or fraudulent review behavior.
    \item \textbf{Yelp-Chi} \cite{dou2020caregnn,qiao2023truncated}: Review graph from YelpChi. Nodes are reviews; relations connect reviews through shared user/product/rating/time information, and our selected operator uses the review--user--review relation. The task is spam review detection.
    \item \textbf{BlogCatalog} \cite{tang_blogcatalog,qiao2023truncated}: Blog author social network with nodes as bloggers and edges as friendships. Node features are derived from blog content; the GAD benchmark uses injected contextual/structural anomalies.
    \item \textbf{Facebook} \cite{conad,qiao2023truncated}: Facebook social network benchmark with users as nodes, friendship edges, and user attributes. Labels mark anomalous users in the benchmark.
    \item \textbf{ACM} \cite{tang_acm,qiao2023truncated}: Academic network derived from ArnetMiner/ACM data, with papers as nodes and content/metadata-derived attributes. The GAD benchmark uses injected contextual/structural anomalies.
    \item \textbf{T-Finance} \cite{tang2022bwgnn}: Financial transaction graph for fraud detection. Nodes are accounts/entities, edges represent transaction relationships, and anomalies correspond to fraudulent actors.
    \item \textbf{Elliptic} \cite{weber2019amlbitcoin_elliptic}: Bitcoin transaction graph with $\sim$200K nodes and $\sim$234K edges. Nodes are transactions; edges are payment flows. Features include local and aggregated transaction attributes. Labels indicate licit vs.\ illicit transactions.
    \item \textbf{Elliptic++} \cite{elmougy2023ellipticplusplus}: Extension of Elliptic that augments Bitcoin transaction data with wallet/address graph information for AML research; we use the transaction-level graph/features in the benchmark.
    \item \textbf{DGraph} \cite{huang2022dgraph}: Large-scale financial graph with $\sim$3.7M nodes and $\sim$4.3M edges from a real-world lending platform. Nodes are users, directed edges encode emergency-contact relationships, and labels identify users with fraudulent borrowing behavior.
\end{itemize}

\begin{table*}[t]
  \centering
  \caption{%
    Raw benchmark statistics for graph datasets used in Table~\ref{tab:main_auroc}.
    \#Feat = feature dimension; Anom.\% = anomaly rate (\%).
  }
  \label{tab:dataset_stats}

  \footnotesize
  \setlength{\tabcolsep}{4pt}
  \renewcommand{\arraystretch}{1.05}

  \begin{tabular}{l r r r r}
    \toprule
    Dataset & \#Nodes & \#Edges & \#Feat & Anom.\% \\
    \midrule
    \multicolumn{5}{l}{\textit{Standard benchmarks (Table~\ref{tab:main_auroc})}} \\
    Weibo       & 8,405   & 407,963   & 400   & 4.1 \\
    Reddit      & 10,984  & 168,016   & 64    & 3.3 \\
    Amazon      & 11,944  & 4,398,392 & 25    & 6.9 \\
    YelpChi     & 45,954  & 3,846,979 & 32    & 4.9 \\
    BlogCatalog & 5,196   & 171,743   & 8,189 & 5.8 \\
    Facebook    & 1,081   & 55,104    & 576   & 2.5 \\
    ACM         & 16,484  & 71,980    & 8,337 & 3.6 \\
    \midrule
    \multicolumn{5}{l}{\textit{Fraud benchmarks (Table~\ref{tab:main_auroc})}} \\
    T-Finance   & 39,357  & 1,222,543 & 10    & 4.6 \\
    Elliptic    & 203,769   & 234,355   & 166 & $\sim$2.2 \\
    Elliptic++  & 203,769   & 234,355   & 182 & $\sim$2.2 \\
    DGraph      & 3,700,550 & 4,300,999 & 17  & 1.3 \\
    \bottomrule
  \end{tabular}
  \vspace{2pt}
  \parbox{0.96\textwidth}{\scriptsize\textit{Note.}
  This table reports the edge count supplied by the benchmark source.
  The selector statistics in Table~\ref{tab:selector_graph_stats} are computed on the graph operator actually used by EB-GAD after deterministic symmetrization, relation/operator selection, affinity reweighting, or transaction aggregation.
  Therefore $\#\mathrm{Edges}/(2n)$ here need not equal the operator-level $d_e$ used by the fixed selector rule; the mismatched cases are due to the selected operators spelled out in Table~\ref{tab:selector_graph_stats}.
  Anomaly rates are computed from the label files as loaded by our pipeline; for Weibo the released PyGOD label file contains 347 anomalous accounts of 8{,}405 (4.1\%), not the 868 (10.3\%) listed in the source documentation, and all methods in Table~\ref{tab:main_auroc} are evaluated against the same released labels.
  Evaluation sets: Amazon and YelpChi are scored on the graphs of the TAM benchmark (10{,}224 nodes, 6.8\% anomalous; 23{,}831 nodes, 5.1\%); Elliptic and Elliptic++ on the 46{,}564 labeled transactions (9.8\% illicit); DGraph on its 1{,}225{,}601 labeled nodes (1.3\%).}
\end{table*}

\subsection{Baseline Models}

Here we provide more details on the baseline models covered in Table~\ref{tab:main_auroc}.

\begin{itemize}
    \item \textbf{LOF} \cite{lof_2000}: Local Outlier Factor, a density-based method that assigns each point a degree of outlierness based on the local density of its neighborhood. Implemented via PYOD \cite{chen2025pyod}.
    \item \textbf{DIF} \cite{DIF}: Deep Isolation Forest, a deep learning extension of Isolation Forest that learns representations for anomaly detection. Implemented via PYOD.
    \item \textbf{ANOMALOUS} \cite{anomalous}: A graph-based method that models anomalies via residual analysis of the adjacency matrix. From the PYGOD library \cite{pygod}.
    \item \textbf{DOMINANT} \cite{ding2019deep}: Deep Anomaly Detection on Attributed Networks using a graph autoencoder with structure and attribute reconstruction. From PYGOD.
    \item \textbf{AnomalyDAE} \cite{fan2020anomalydae}: Dual autoencoder for attributed-network anomaly detection, jointly modeling structure and attribute reconstruction. From PYGOD.
    \item \textbf{CONAD} \cite{conad}: Contrastive self-supervised learning for graph anomaly detection. From PYGOD.
    \item \textbf{CoLA} \cite{Liu_2022_cola}: Contrastive Learning for Anomaly detection on graphs; leverages self-supervised contrastive objectives. From PYGOD.
    \item \textbf{TAM} \cite{qiao2023truncated}: Truncated Affinity Maximization, a one-class homophily method that learns local-affinity-based representations for graph anomaly detection.
    \item \textbf{DiffGAD} \cite{li2024diffgad}: Diffusion-based graph anomaly detector that injects discriminative content into latent representations and uses diffusion/reconstruction signals for anomaly scoring.
\end{itemize}

\subsection{Native-Score Evaluation and Orientation Audit}
\label{app:native_score_orientation}

Table~\ref{tab:main_auroc} evaluates every unsupervised baseline in its native anomaly-score direction, as fixed by the official implementation or paper: PyOD scores use \texttt{decision\_scores\_}, PyGOD scores use \texttt{decision\_score\_}, TAM uses its published local-affinity anomaly score, and DiffGAD uses its reconstruction/diffusion anomaly score. This convention is important for a label-free comparison. Reversing a score after observing test labels is an oracle operation, so it is not used in the main table even when a native score is anti-correlated with the benchmark labels.

To check whether the sub-50 entries in Table~\ref{tab:main_auroc} reflect missing data or merely native-score anti-correlation, Table~\ref{tab:score_orientation_audit} reruns the threatening PyOD/DiffGAD fraud cases and reports both the native AUROC and the diagnostic oracle value obtained by negating the scores. The flipped column is not a competing baseline; it only shows how much a label-based score reversal would change the ranking. In every audited case, normal nodes receive larger native scores on average than anomalous nodes, so the sub-50 AUROCs are genuine anti-correlations of the native baseline scores.

\begin{table}[t]
\centering
\scriptsize
\caption{Score-orientation audit for sub-50 baselines whose oracle sign reversal could affect the fraud-dataset comparison. Native AUROC uses the method's fixed label-free anomaly-score direction and is the only value used in Table~\ref{tab:main_auroc}. Oracle-flipped AUROC is computed from the negated score after labels are known and is diagnostic only.}
\label{tab:score_orientation_audit}
\setlength{\tabcolsep}{5pt}
\begin{tabular}{l l c c c}
\toprule
Dataset & Baseline & Native AUROC & Oracle-flipped AUROC & Mean score normal/anomaly \\
\midrule
Elliptic   & DIF     & $22.1{\pm}0.9$ & $77.9{\pm}0.9$ & $0.335/0.303$ \\
Elliptic++ & DIF     & $22.6{\pm}0.5$ & $77.4{\pm}0.5$ & $0.334/0.303$ \\
Elliptic   & DiffGAD & $27.1{\pm}0.4$ & $72.9{\pm}0.4$ & $20.90/13.31$ \\
Elliptic++ & DiffGAD & $27.9{\pm}7.8$ & $72.1{\pm}7.8$ & $967.09/417.67$ \\
\bottomrule
\end{tabular}
\end{table}

\subsection{Selector and Hyperparameter Specification}
\label{app:selector_details}

\paragraph{Scope.}
Main \S\ref{subsec:prior_selection} introduces the selector at the level of principles: four score families tied to graph statistics, a quality plus stability score $A(r)$, and the role of EB in fitting $(\rho,\kappa)$.
This section states the rule exactly as it is run; the released code contains the same implementation and the script that regenerates Table~\ref{tab:main_auroc}.
No rule takes a label as input.
The cutoffs are implementation defaults, identical on all 11 datasets, not derived quantities: they were fixed during the development of the method on the benchmark families studied here, so they are part of its engineering. Table~\ref{tab:provenance} classifies every constant, and Appendix~\ref{app:selector_studies} reports what happens when the cutoffs are perturbed.

\paragraph{Preprocessing, spectrum and grids.}
Features are z-scored. For the equilibrium and profile families, features with $D_x>100$ are projected on their first 64 principal components; PCA is a deterministic unsupervised linear map, so the closed-form results apply in the projected space.
The full Laplacian spectrum is used when $n<20{,}000$; larger graphs use the $k$ smallest eigenpairs with $k=500$ ($n<10^5$), $300$ ($n<5{\cdot}10^5$), $200$ ($n<10^6$) or $128$. Dense eigendecompositions run in double precision, and LOBPCG bases are orthonormalized and Ritz-refined in double precision. All computed modes, including Laplacian null modes, are kept in the modeled subspace (Appendix~\ref{app:modeling_details}).
The bandwidth grid holds the three values of $\{0.1,0.2,0.5,0.7,1,2,5\}$ that are closest on a log scale to $\gamma_0=\operatorname{clip}((0.5+h)/\sqrt{\max(d_e,10)/10},0.05,10)$ among those not larger than $2\gamma_0$ (DGraph, whose spectrum is computed once and cached, uses $\{0.5,1,2\}$). The run and its finite-horizon families are built at $\gamma^\star$, the grid value with the largest residual-space profile likelihood. The equilibrium anchor is computed at $\gamma_a$, the grid value closest to $\gamma_0$ on the log scale, with the EB fit of $(\rho,\kappa)$ at $\gamma_a$ (Table~\ref{tab:gamma_profile}; Appendix~\ref{app:regeneration} gives the reason and the alternatives).
The other grids are $\kappa\in\{0,0.001,0.003,0.01,0.03,0.1,0.3,0.5,1,2,3,5,7,10,15,20\}$, $\rho\in[0,1]$ by a bounded one-dimensional search (clamped Newton or Brent's method) with an endpoint check (Remark~\ref{rem:landscape}), and horizons $\Gamma\in\{0.02,0.05,0.1,0.2,0.5,1,2,5,10,50,\infty\}$; the control-energy and control-ratio pools use $\Gamma\in\{0.25,0.5,1,2,3,4,8,12\}$ and $\lambda_c\in\{0.3,0.5,1,2,5,10,50,100,500\}$.

\paragraph{Affinity template.}
The ``affinity'' template in Table~\ref{tab:ebgad_hyperparams} keeps the edge set of the graph operator fixed but multiplies features by deterministic node reliability weights before smoothing,
\[
m_{\rm aff}=S_{\gamma,\nu}WX,\qquad W=\operatorname{diag}(w_1,\ldots,w_n),
\]
where $w_i\in[0,1]$ is the min--max normalized local affinity, $a_i=\deg(i)^{-1}\sum_{j\in N(i)}\cos(z_i,z_j)$. When PCA is enabled, it is applied right after z-scoring and before the affinity weights. The Laplacian is that of the input graph in every run of Table~\ref{tab:main_auroc}.

\paragraph{Selector statistics.}
The graph statistics $(h,d_e,D_x)$ are computed from the graph operator and the node features as provided (no class labels enter $h$ or any selector diagnostic):
\[
h=\frac{1}{|E_s|}\sum_{(i,j)\in E_s}\frac{x_i^\top x_j}{\|x_i\|_2\|x_j\|_2+\epsilon}, \qquad d_e=|E_{\rm dir}|/(2n),
\]
where $E_s$ is a fixed sample of up to $10^5$ edges. The operator-level $d_e$ in Table~\ref{tab:selector_graph_stats} is not recomputed from the raw edge column in Table~\ref{tab:dataset_stats}: Amazon uses the filtered U--P--U user--product--user relation ($n{=}10{,}224$, $|E_{\rm dir}|{=}351{,}216$); YelpChi uses the filtered R--U--R review--user--review relation ($n{=}23{,}831$, $|E_{\rm dir}|{=}98{,}630$); BlogCatalog reports undirected raw edges but the operator is bidirectional ($|E_{\rm dir}|{=}345{,}566$); ACM uses the symmetrized input graph ($|E_{\rm dir}|{=}164{,}350$); T-Finance uses the BWGNN-released homogenized transaction graph \cite{tang2022bwgnn} ($|E_{\rm dir}|{=}42{,}445{,}086$); and Weibo, Reddit, Facebook, Elliptic, Elliptic++, DGraph use raw conventions. For DGraph, $h$ and $d_e$ are computed on the full graph and the spectrum on the subgraph of its labeled nodes ($k=128$, $n=1{,}225{,}601$). $h$ saturates near one for correlated non-negative features (e.g., Reddit), so it enters only through coarse cutoffs.

\paragraph{Ordered rules.}
The selector chooses among four score families: equilibrium anchors ($J^\star$ or $R$), control-energy scores ($C_{\Gamma,\lambda_c}$), control-ratio scores ($CR_{\Gamma,\lambda_c}$), and profile aggregation over the finite-horizon bank.
The rules are evaluated in order and the first that applies is used:
\begin{enumerate}[leftmargin=*, itemsep=1pt, topsep=2pt]
\item If $d_e<1$, $D_x\le256$, and $k/n\le5\cdot10^{-3}$, use the equilibrium ratio anchor $R$.
\item Else if $h<0.05$, $D_x>1000$, and $d_e\ge20$, use the control-energy family on the affinity template.
\item Else if $0.05\le h<0.25$, $D_x>1000$, and $d_e<10$, use the control-ratio family on the affinity template.
\item Else if $d_e\ge100$, use profile aggregation (trajectory branch); in the intermediate dense case $10\le d_e<100$, this family is used only if the dense profile diagnostic passes and the anchor guard is inactive.
\item Else if $h\ge0.50$ and $D_x\le128$: if the sparse-tail diagnostic passes, use profile aggregation (sparse-tail branch); otherwise use the control-energy family when $d_e\ge10$ and the equilibrium anchor when $d_e<10$.
\item Else if $0.25\le h<0.50$, $128<D_x\le1000$, $d_e\ge20$, and the scale-entropy diagnostic passes, use profile aggregation (scale-entropy branch).
\item Otherwise use the equilibrium anchor.
\end{enumerate}
The diagnostics are unlabeled. The \emph{dense profile} diagnostic requires $\mathrm{NullKS}(J^\star)\ge0.50$ and a Berk--Jones statistic \cite{berk1979goodness} of the profile $p$-values in $[20,120]$. The \emph{anchor guard} is active when NullKS selects $J^\star$ with $\mathrm{NullKS}(J^\star)\ge0.60$. The \emph{sparse-tail} diagnostic requires Berk--Jones in $[20,100]$, equilibrium NullKS below $0.20$, and profile NullKS below $0.35$. The \emph{scale-entropy} diagnostic requires a two-groups mixture mass $\pi\in[0.50,0.95]$ and that the tail-stability and graph-tail criteria both pick the component-entropy score.

\paragraph{Within-family score.}
\emph{Equilibrium anchor}: $J^\star$ or $R$ at the anchor bandwidth $\gamma_a$, whichever has the larger NullKS there (ties go to $J^\star$; rule 1 fixes $R$). The diagnostics that the ordered rules read are those of the EB-fitted run at $\gamma^\star$.
\emph{Profile aggregation}: the scale-entropy branch uses the entropy of the two-groups component responsibilities of the path $p$-values; the sparse-tail branch uses the posterior of a two-groups mixture over the $p$-value bank; the trajectory branch uses the path band with the largest NullKS.
\emph{Control energy and control ratio}: the admissible pool is the set of candidates of the family on the rule's template (rule 2: affinity template, $\gamma\in[0.12,0.25]$, PCA dimension $\lceil 256\cdot5/d_e\rceil$ rounded up to the grid $\{16,24,32,48,64\}$, $\Gamma\ge1$, $\lambda_c\le1$, NullKS in $[0.55,0.85]$, Gini coefficient of the score in $[0.70,0.90]$; rule 3: affinity template, $\gamma\in[0.075,0.20]$, PCA 256, $\Gamma\le0.5$, $\lambda_c\le0.5$, Gini at most $0.05$; rule 5: low-pass template, $\gamma=5$, no PCA, NullKS in $[0.20,0.50]$, Gini in $[0.25,0.70]$). In rule 5 the two pool members with the largest NullKS are rank-averaged. In rules 2 and 3 the members are ranked by $A(r)$ (stability plus a $0.1$-weighted quality term) and the top three are rank-averaged; if the pool is empty, all control-energy candidates are used.

\paragraph{Null model and diagnostic aggregation.}
Under the full-spectrum Gaussian model with fixed precision $Q_\rho$, $[Q_\rho^{1/2}\Delta]_{i,:}\sim\mathcal{N}(0,I_D)$ and so $2J_i^\star\sim\chi^2_D$ exactly. A finite-horizon member is a quadratic form with weights $c_j(\Gamma,\lambda_c)$ in place of $q_j$; under the same fitted law its null is a weighted chi-squared with known weights $c_j/q_j$. After EB fitting or spectral truncation these laws are used as plug-in nulls. For selection we summarize each empirical score distribution by a moment-matched scaled chi-squared $a_S\cdot\chi^2_{\hat k_S}$ and compute
\begin{equation}
\mathrm{NullKS}(S) := \mathrm{KS}\!\bigl(S/\hat a_S,\; \chi^2_{\hat k_S}\bigr),\qquad
\hat k_S=2\hat\mu_S^2/\hat\sigma_S^2,\quad \hat a_S=\hat\sigma_S^2/(2\hat\mu_S).
\label{eq:ks_score_select}
\end{equation}
Larger NullKS indicates stronger unlabeled deviation from a chi-squared bulk. Each scalar diagnostic is converted to a percentile rank within the admissible pool before being combined with stability terms, making KS distances and Spearman-rank stability commensurate; percentile ranks average ties. For profile aggregation, scores produce nodewise $p$-values $p_{ir}$ that are combined by ACAT \cite{liu2020cauchy}, minimum-$p$, mean $-\log p_{ir}$, or a two-groups mixture over the $p$-value bank. When $\kappa$ is fitted at the upper end of its grid the plug-in $p$-values can saturate to a constant; constant members carry no ranking information and, with averaged ties, cannot contribute any.

\begin{table*}[!htb]
  \centering
  \caption{%
    Label-free EB-GAD configuration behind each cell of Table~\ref{tab:main_auroc}.
    \emph{Template}: low-pass Mat\'ern or feature-affinity-weighted (\S\ref{app:selector_details}).
    $k$: number of retained low-frequency Laplacian eigenpairs (``all'' = full spectrum).
    $\rho,\kappa,\gamma$: EB-fitted graph trust and Mat\'ern inverse length-scale, and the template bandwidth of the selected score (the anchor bandwidth $\gamma_a$ for an anchor, the likelihood-selected bandwidth for a profile, the pool rule for control energy and ratio).
    \emph{PCA}: retained PCA dimensions. \emph{Rule}: the ordered rule of \S\ref{app:selector_details} that fires.
  }
  \label{tab:ebgad_hyperparams}
  \scriptsize
  \setlength{\tabcolsep}{3pt}
  \renewcommand{\arraystretch}{1.05}
  \resizebox{\textwidth}{!}{%
  \begin{tabular}{l c c c c c c c l}
    \toprule
    Dataset & Template & $k$ & $\rho$ & $\kappa$ & $\gamma$ & PCA & Rule & Selected family (score) \\
    \midrule
    Weibo & low-pass & all & 1.000 & 1 & 0.5 & 64 & 7 & equilibrium anchor ($J^\star$) \\
    Reddit & low-pass & all & 1.000 & 1 & 2 & no & 5 & equilibrium anchor ($R$) \\
    Amazon & low-pass & all & 0.458 & 0.5 & 5 & no & 5 & control energy (top-2 $C_{\Gamma,\lambda_c}$ fusion) \\
    YelpChi & low-pass & 500 & 1.000 & 20 & 1 & no & 5 & equilibrium anchor ($J^\star$) \\
    BlogCatalog & affinity & all & 1.000 & 1 & 0.18 & 48 & 2 & control energy (stability-ranked $C_{\Gamma,\lambda_c}$ fusion) \\
    Facebook & low-pass & all & 0.727 & 3 & 0.7 & 64 & 6 & profile aggregation (scale-entropy) \\
    ACM & affinity & all & 1.000 & 0.01 & 0.15 & 256 & 3 & control ratio (stability-ranked $CR_{\Gamma,\lambda_c}$ fusion) \\
    Elliptic & low-pass & 300 & 1.000 & 20 & 1 & 64 & 1 & equilibrium anchor ($R$) \\
    Elliptic++ & low-pass & 300 & 0.933 & 0.1 & 2 & 64 & 1 & equilibrium anchor ($R$) \\
    DGraph & low-pass & 128 & 1.000 & 20 & 1 & no & 1 & equilibrium anchor ($R$) \\
    T-Finance & low-pass & 500 & 1.000 & 0.003 & 0.2 & no & 4 & profile aggregation (trajectory) \\
    \bottomrule
  \end{tabular}
  }
\end{table*}

\begin{table*}[!htb]
\centering
\scriptsize
\caption{Unlabeled quantities read by the selector. $h$: cosine feature homophily over a fixed sample of edges, $d_e$: edges per node of the graph operator, $D_x$: raw feature dimension, $k/n$: truncation ratio. Diagnostics of the EB-fitted run: NullKS of its two equilibrium scores, Berk--Jones statistic (BJ) and NullKS of the profile score, and the two-groups mixture mass $\pi$.}
\label{tab:selector_graph_stats}
\setlength{\tabcolsep}{4pt}
\begin{tabular}{l c c c c c c c c c}
\toprule
Dataset & $h$ & $d_e$ & $D_x$ & $k/n$ & NullKS($J^\star$) & NullKS($R$) & BJ & NullKS(profile) & $\pi$ \\
\midrule
    Weibo & 0.068 & 24.27 & 400 & all & 0.663 & 0.156 & 78.5 & 0.258 & 0.786 \\
    Reddit & 0.993 & 7.65 & 64 & all & 0.945 & 0.138 & 57.6 & 0.206 & 0.878 \\
    Amazon & 0.644 & 17.18 & 25 & all & 0.277 & 0.071 & 236.2 & 0.208 & 0.740 \\
    YelpChi & 0.876 & 2.07 & 32 & $2.10{\times}10^{-2}$ & 0.000 & 0.000 & 1225.2 & 0.334 & 0.024 \\
    BlogCatalog & 0.009 & 33.25 & 8189 & all & 0.673 & 0.089 & 169.9 & 0.246 & 0.077 \\
    Facebook & 0.375 & 25.49 & 576 & all & 0.036 & 0.044 & 8.3 & 0.227 & 0.811 \\
    ACM & 0.143 & 4.99 & 8337 & all & 0.492 & 0.074 & 24.4 & 0.253 & 0.968 \\
    Elliptic & 0.751 & 0.58 & 166 & $1.47{\times}10^{-3}$ & 0.853 & 0.909 & 1810.7 & 0.289 & 0.024 \\
    Elliptic++ & 0.978 & 0.58 & 182 & $1.47{\times}10^{-3}$ & 0.874 & 0.905 & 1987.2 & 0.289 & 0.024 \\
    DGraph & 0.408 & 0.58 & 17 & $1.04{\times}10^{-4}$ & 0.000 & 0.000 & 0.0 & 0.000 & 0.024 \\
    T-Finance & 0.817 & 539.23 & 10 & $1.27{\times}10^{-2}$ & 0.203 & 0.557 & 578.5 & 0.205 & 0.996 \\
\bottomrule
\end{tabular}
\end{table*}

\begin{table}[!htb]
\centering
\scriptsize
\caption{Provenance of every quantity that enters a score or the selection. \emph{Model-derived}: follows from the fitted model. \emph{Canonical}: a standard statistic or combiner, with citation. \emph{Default}: an implementation default, fixed once and identical on all datasets, with the sensitivity study that covers it.}
\label{tab:provenance}
\setlength{\tabcolsep}{4pt}
\begin{tabular}{p{0.34\linewidth} p{0.14\linewidth} p{0.44\linewidth}}
\toprule
Quantity & Class & Source or tested range \\
\midrule
$(\rho,\kappa)$; $\gamma$ within its grid & model-derived & maximum of the residual-space likelihood (Alg.~\ref{alg:prior_opt}) \\
$J^\star$, $R$, $C_{\Gamma,\lambda_c}$, $CR_{\Gamma,\lambda_c}$ & model-derived & closed forms of \S\ref{subsec:anomaly_score} \\
plug-in nulls of the scores & model-derived & chi-squared laws under the fitted prior \\
Kolmogorov--Smirnov, Berk--Jones & canonical & \cite{berk1979goodness} \\
ACAT, minimum-$p$, mean $-\log p$ combiners & canonical & \cite{liu2020cauchy}; Fisher's method \\
two-groups mixture of $p$-values & canonical & \cite{efron2012large} \\
rank correlation as stability $S(r)$ & canonical & Spearman correlation, average ties \\
grids for $\gamma$, $\kappa$, $\Gamma$, $\lambda_c$ & default & listed in this section \\
PCA dimension, truncation $k$ by graph size & default & listed in this section \\
graph-statistic cutoffs on $h$, $d_e$, $D_x$, $k/n$ & default & $\pm25\%$ and $\pm50\%$ (Table~\ref{tab:cutoff_sensitivity}) \\
diagnostic windows (NullKS, Berk--Jones, $\pi$) & default & $\pm25\%$ and $\pm50\%$ (Table~\ref{tab:cutoff_sensitivity}) \\
pool windows of rules 2, 3, 5 (NullKS, Gini) & default & listed in this section; not perturbed \\
weights inside $A(r)$ & default & ingredient ablation (Table~\ref{tab:ar_ablation}) \\
\bottomrule
\end{tabular}
\end{table}

\begin{figure*}[t]
    \centering
    \includegraphics[width=0.75\linewidth]{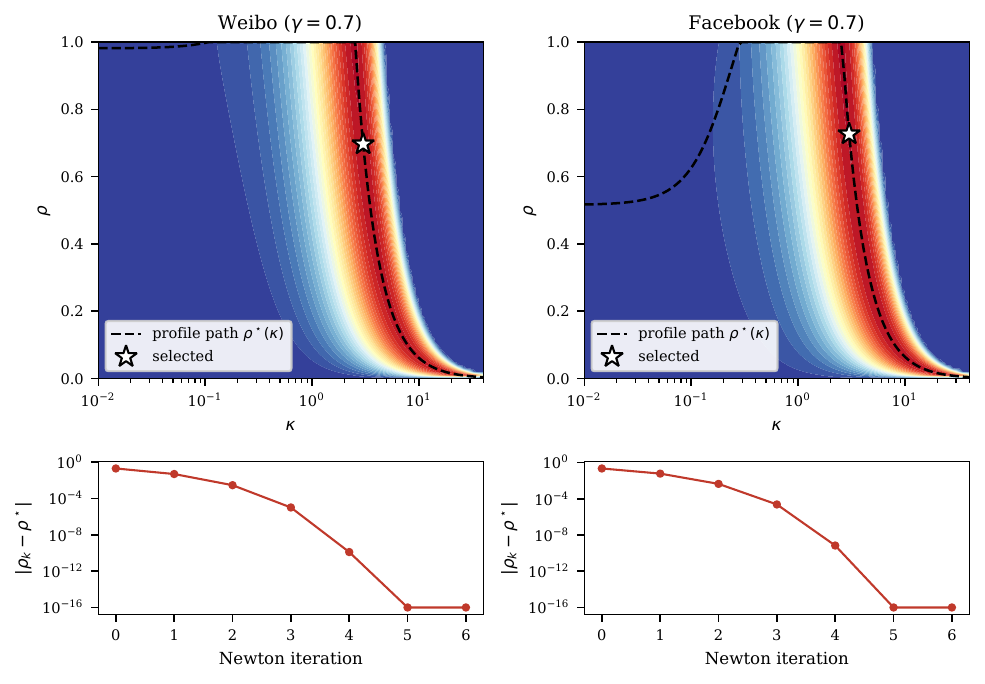}
    \caption{%
        \textbf{Top:} residual-space log-likelihood over $(\rho, \kappa)$ of the EB-fitted run (likelihood-selected bandwidth $\gamma^\star=0.7$) for Weibo and Facebook (values below the 35th percentile are clipped for display).
        Dashed curve: profile path $\rho^\star(\kappa)$; star: selected optimum on the $\kappa$ grid.
        The likelihood is concave in $\rho$ (Remark~\ref{rem:landscape}) with dataset-dependent structure in $\kappa$.
        \textbf{Bottom:} Newton iterates for $\rho$ at the selected $\kappa^\star$.
    }
    \label{fig:ml_contour}
\end{figure*}

\begin{figure}[t]
    \centering
    \includegraphics[width=\linewidth]{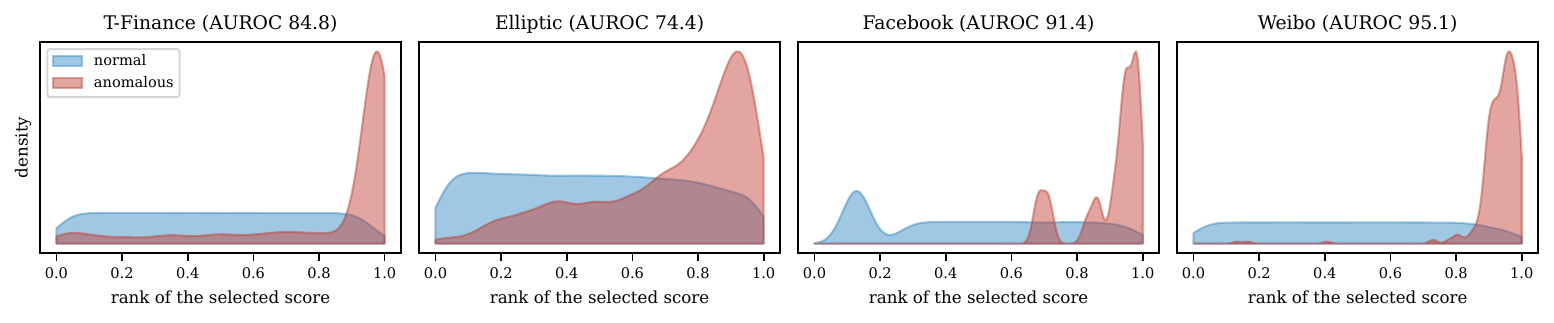}
    \caption{Density of the selected EB-GAD score for normal and anomalous nodes on four datasets (original node order). Scores are shown on their rank scale, so a method at chance would give two flat densities.}
    \label{fig:density_sep}
\end{figure}

\subsection{Equilibrium Scores}
\label{app:equilibrium}

Table~\ref{tab:ks_ablation} reports the two equilibrium scores at the anchor bandwidth and the one that NullKS selects without labels. For the six datasets routed to an anchor this is the entry of Table~\ref{tab:main_auroc}; for the other five, the comparator of main Table~\ref{tab:bank_ablation} is the equilibrium score at the configuration of the selected score. NullKS picks the better of the two on Weibo, Amazon, BlogCatalog and ACM, and the worse on Reddit, Facebook and T-Finance; on Facebook the difference is large ($47.4$ against $89.7$) and is the cost that the cutoff study of Appendix~\ref{app:selector_studies} reports. On YelpChi and DGraph the fitted $\kappa$ is at the upper end of its grid, both NullKS values are zero, and the tie goes to $J^\star$ (DGraph, like Elliptic and Elliptic++, is assigned $R$ by rule 1).

\begin{table}[!htb]
\centering
\scriptsize
\caption{Equilibrium anchor: AUROC (\%) of $J^\star$ and $R$ at the anchor bandwidth $\gamma_a$ (the grid value nearest to $\gamma_0$) with the EB fit of $(\rho,\kappa)$ at that bandwidth. \textbf{Bold}: the score chosen without labels (larger NullKS; ties go to $J^\star$; rule 1 of \S\ref{app:selector_details} fixes $R$ on sparse truncated graphs). Last row: the score selected in Table~\ref{tab:main_auroc}.}
\label{tab:ks_ablation}
\setlength{\tabcolsep}{2pt}
\begin{tabular}{l ccccccccccc}
\toprule
& Weibo & Reddit & Amazon & YelpChi & BlogCat. & Facebook & ACM & Elliptic & Ell.++ & DGraph & T-Fin. \\
\midrule
$\gamma_a$ & 0.5 & 2 & 1 & 1 & 0.2 & 0.5 & 0.7 & 1 & 2 & 1 & 0.2 \\
$J^\star$ & \textbf{95.1} & 62.0 & \textbf{75.8} & \textbf{71.7} & 54.0 & \textbf{47.4} & \textbf{75.6} & 58.3 & 30.2 & 60.5 & 81.6 \\
$R$       & 47.3 & \textbf{60.6} & 48.2 & 64.6 & \textbf{61.0} & 89.7 & 72.1 & \textbf{74.5} & \textbf{72.5} & \textbf{66.8} & \textbf{80.1} \\
\midrule
Table~\ref{tab:main_auroc} & 95.1 & 60.6 & 78.0 & 71.7 & 78.6 & 91.4 & 86.0 & 74.5 & 72.5 & 66.8 & 84.8 \\
\bottomrule
\end{tabular}
\end{table}

\subsection{Resource Utility and Scalability Measures}

Baselines were trained and evaluated with their reported hyperparameters on NVIDIA A100 40GB GPUs. The EB-GAD results of Table~\ref{tab:main_auroc} were regenerated on CPU. For medium and large graphs we use a truncated eigendecomposition; the retained $k$ for each dataset is reported in Table~\ref{tab:ebgad_hyperparams}.

To scale EB-GAD to large graphs, we adopt the following efficiency measures:
\begin{itemize}
    \item \textbf{Truncated eigendecomposition:} For graphs with $n \geq 20$K nodes we use the Locally Optimal Block Preconditioned Conjugate Gradient (LOBPCG \cite{LOBPCG_2001}) method, with a seeded initial block, to compute the $k$ smallest eigenpairs of the sparse Laplacian; the basis is then orthonormalized and Ritz-refined in double precision. These are the smooth graph modes used by the low-pass template and the Mat\'ern precision; unretained rough modes are omitted from the likelihood and the control-energy numerator, while ratio scores still normalize by the full nodewise residual norm. The eigenvector matrix is $(n \times k)$ rather than $(n \times n)$, reducing memory from $O(n^2)$ to $O(nk)$; $k$ is scaled by graph size (Table~\ref{tab:ebgad_hyperparams}).
    \item \textbf{Eigenspace operations:} All graph filters and prior operations use products $V (\mathbf{f}(\lambda) \odot (V^\top \mathbf{x}))$ in $O(n \cdot k \cdot D)$ per step, avoiding dense matrix exponentials.
    \item \textbf{Sparse Laplacian:} For graphs with $n \geq 8$K nodes, we use a sparse construction to avoid materializing dense $n \times n$ matrices.
    \item \textbf{Eigen cache:} The dense precision matrix $Q$ is skipped for large graphs; all scoring operations use the precomputed eigen cache instead.
\end{itemize}

\subsection{Computation Time}

Table~\ref{tab:runtime} reports wall-clock times. The baseline times are training plus inference on one NVIDIA A100 GPU, as measured for the submitted version. The EB-GAD times were measured for this version on 16 CPU threads of a shared server and cover the complete run behind a cell of Table~\ref{tab:main_auroc}: loading the data, one double-precision eigendecomposition per template bandwidth of the grid, the EB fits, every score family that the run computes, and the label-free selection. They were measured before the anchor bandwidth $\gamma_a$ was introduced; the anchor reuses the residual that the run computes at $\gamma_a$, a value of its grid, and adds one EB fit and two scores. The times are upper bounds for a deployment, which would cache one eigendecomposition and compute one family. The two sides use different hardware, so the table indicates orders of magnitude and is not a controlled comparison.
EB-GAD takes $8$ seconds on Facebook and $33$ minutes on DGraph, where the eigendecomposition of the labeled subgraph accounts for $30$ minutes. The candidate-pool datasets (Amazon, BlogCatalog, ACM) take $11$ to $26$ minutes because every candidate configuration recomputes the eigendecomposition; the implementation also recomputes it for every template bandwidth although the Laplacian does not depend on the bandwidth, so caching it would reduce all times.

\begin{table*}[!htb]
  \centering
  \caption{%
    Wall-clock time (seconds). Baselines: training and inference on one NVIDIA A100 GPU, as measured for the submitted version. EB-GAD: the complete run behind its cell of Table~\ref{tab:main_auroc} on 16 CPU threads, from loading the data to the selected score: eigendecompositions in double precision for every template bandwidth of the grid, EB fits, all score families of the run and the label-free selection. The two columns use different hardware and are not a controlled comparison. \oom{} = out of memory.
  }
  \label{tab:runtime}
  \scriptsize
  \setlength{\tabcolsep}{3pt}
  \renewcommand{\arraystretch}{1.05}
  \begin{tabular}{l rrrrrrr r}
    \toprule
    Dataset & ANOM. & DOMINANT & AnomDAE & CONAD & CoLA & DiffGAD & TAM & EB-GAD (CPU) \\
    \midrule
    Weibo & 78 & 63 & 1 & 95 & 130 & 134 & 858 & 84 \\
    Reddit & 38 & 35 & 2 & 57 & 8 & 110 & 5,467 & 136 \\
    Amazon & 25 & 37 & 153 & 308 & 14 & 106 & 10,739 & 708 \\
    YelpChi & 170 & 185 & 1,152 & 1,359 & 5 & 118 & 5,106 & 27 \\
    BlogCatalog & 447 & 430 & 92 & 3,047 & 18 & 3,272 & 8,764 & 675 \\
    Facebook & 3 & 20 & 9 & 56 & 3 & 118 & 479 & 8 \\
    ACM & 6,503 & 290 & 580 & 2,293 & 26 & 39,209 & 14,446 & 1,575 \\
    Elliptic & \oom & 19,802 & \oom & \oom & 263 & 294 & 1,176 & 1,101 \\
    Elliptic++ & \oom & 16,962 & \oom & \oom & 288 & 292 & 1,186 & 1,429 \\
    DGraph & \oom & \oom & \oom & \oom & 4,117 & \oom & 3,929 & 1,956 \\
    T-Finance & 963 & 3,425 & 9,761 & 5,028 & 1,107 & 384 & \oom & 1,471 \\
    \bottomrule
  \end{tabular}
\end{table*}

\subsection{Dimensionality Reduction and Nonlinear Extensions}
\label{app:encoder}

The main text uses PCA as a training-free dimensionality reduction that preserves the linear-Gaussian structure.
Here we discuss the theoretical landscape of dimensionality reduction for EB-GAD and report results with trained nonlinear encoders.

\paragraph{Linear reductions preserve the Gaussian structure.}
Feature-space linear maps preserve Gaussianity. If $z=(I_n\otimes W)x$ and $x \sim \mathcal{N}(\mathbf{m}, Q_\rho^{-1}\otimes I_D)$, then
$z \sim \mathcal{N}((I_n\otimes W)\mathbf{m}, Q_\rho^{-1}\otimes WW^\top)$.
For PCA with orthonormal retained components (and after standardization), $WW^\top=I_p$, so the same closed-form anomaly scores and residual-space likelihood apply exactly in the projected space.
PCA is deterministic, unsupervised, and maximizes retained variance among linear projections.

\paragraph{Nonlinear reductions break the Gaussian structure.}
When a nonlinear encoder $f_\theta: \mathbb{R}^D \to \mathbb{R}^p$ is applied, the latent features $z = f_\theta(x)$ are no longer Gaussian even if $x$ is.
The residual-space likelihood objective is no longer available in closed form, and a variational bound is needed for prior selection.

\paragraph{Variational formulation (ELBO).}
For this nonlinear-encoder extension only, one could construct a terminally conditioned GOU bridge via Doob's $h$-transform on a discrete time grid and optimize a reconstruction-plus-KL variational objective in the spirit of learned diffusion-bridge/reverse-transition models \cite{zhou2024ddbm,yue2023goub,zhu2025unidb}. The bound would combine a per-step KL between the analytic bridge transition $q_\varphi$ and the learned reverse transition $p_\theta$ with an endpoint reconstruction term; in the linear-Gaussian case (linear encoder or no encoder), the bridge transitions match the true posterior exactly, the variational gap vanishes, and the bound reduces to the exact residual-space objective \eqref{eq:marginal_likelihood}.
This bridge is not used in the main EB-GAD pipeline, which relies on the unconditioned GOU transition law; we mention the ELBO only to indicate how a future nonlinear-encoder extension would inherit the empirical-Bayes structure.

\paragraph{Non-neural nonlinear options.}
Beyond neural encoders, several non-neural nonlinear methods could extend the framework:
\begin{itemize}[nosep]
\item \emph{Kernel PCA}: maps features to a reproducing kernel Hilbert space via a kernel function (e.g., RBF). Non-parametric and training-free, but the infinite-dimensional feature space requires truncation, and the Gaussian assumption is approximate.
\item \emph{Random Fourier Features} \cite{rahimi2007random}: approximates kernel PCA with explicit features $z(x) = \cos(Wx + b)$ for random $W$. Deterministic, finite-dimensional, and training-free, but the nonlinear map makes the likelihood objective approximate.
\item \emph{GPLVM} \cite{lawrence2005probabilistic}: a principled Bayesian nonlinear latent variable model with a natural variational formulation. Requires optimization of variational parameters, making it no longer training-free.
\end{itemize}
These represent a spectrum from fully closed-form (PCA) to fully variational (GPLVM), trading tractability for expressive power.
We leave exploration of these extensions to future work.

\paragraph{Empirical comparison: training-free vs.\ trained encoders.}
Table~\ref{tab:encoder_comparison} compares the training-free results of Table~\ref{tab:main_auroc} against the better of two trained nonlinear encoders (MLP or GraphSAGE) for each dataset; choosing the better encoder by AUROC favors the trained side.
Trained encoders are applied as a preprocessing step: the encoder is trained separately, and the same score formula is then applied to the latent features.
\emph{Training setup:} Adam optimizer (learning rate $10^{-3}$, weight decay $5\!\cdot\!10^{-4}$), 200 epochs with early stopping on the unsupervised reconstruction loss, batch size 1024 (full-batch on small graphs), 2 hidden layers of width 64 with ReLU; for GraphSAGE, mean-aggregator with neighbor sampling sizes $[25,10]$; for the small-encoder retries, hidden width 16 and 1 hidden layer.
For T-Finance and DGraph, we use a small feature-only retry to avoid dense structure reconstruction; the DGraph row reports the best of seven completed small-encoder configurations.
The trained-encoder numbers are those of the submitted version; they were not rerun.

\begin{table}[!htb]
\centering
\caption{Training-free EB-GAD (Table~\ref{tab:main_auroc}) against the better of two trained encoders (MLP, GraphSAGE), chosen per dataset by AUROC, followed by the same score formula. AUROC (\%); gap = training-free minus trained.}
\label{tab:encoder_comparison}
\scriptsize
\begin{tabular}{lccc}
\toprule
Dataset & Training-free & Best trained encoder & Gap \\
\midrule
Weibo & 95.1 & \textbf{95.7} (MLP) & $-0.6$ \\
Reddit & \textbf{60.6} & 58.9 (MLP) & $+1.7$ \\
Amazon & \textbf{78.0} & 60.3 (MLP) & $+17.7$ \\
YelpChi & \textbf{71.7} & 68.0 (SAGE) & $+3.7$ \\
BlogCatalog & \textbf{78.6} & 77.6 (MLP) & $+1.0$ \\
Facebook & 91.4 & \textbf{95.4} (MLP) & $-4.0$ \\
ACM & \textbf{86.0} & 63.7 (MLP) & $+22.3$ \\
Elliptic & \textbf{74.5} & 56.4 (SAGE) & $+18.1$ \\
Elliptic++ & \textbf{72.5} & 57.7 (SAGE) & $+14.8$ \\
DGraph & \textbf{66.8} & 58.8 (SAGE-small) & $+8.0$ \\
T-Finance & 84.8 & \textbf{84.9} (SAGE-small) & $-0.1$ \\
\bottomrule
\end{tabular}
\end{table}

\paragraph{Reading.}
The training-free pipeline is ahead on 8 of 11 datasets, by $1.0$ to $22.3$ points (Reddit, Amazon, YelpChi, BlogCatalog, ACM, Elliptic, Elliptic++, DGraph). A trained encoder is ahead on Weibo and Facebook, by $0.6$ and $4.0$ points, and on T-Finance by $0.1$.
The encoder helps on two social graphs and hurts on five of the six fraud graphs, on ACM, on BlogCatalog and on Reddit.
Where it hurts, two causes are plausible: a nonlinear encoder breaks the Gaussian structure, so the exact residual-space objective \eqref{eq:marginal_likelihood} no longer applies to the latent features, and it can distort the relation between graph frequencies and feature variation on which the scores operate.
The training-free pipeline is also deterministic given the eigenbasis, whereas trained encoders add seed-dependent variance.

\section{Regenerated Results and Additional Studies}
\label{app:studies}

Every EB-GAD number in this paper, in the main tables and in the studies below, was produced for this version by one pipeline. Wherever a score is selected, it is selected without labels; ``best member'' columns report the best score of a pool by AUROC and are diagnostics, not results. The studies of the finite-horizon bank (isolation, contamination, misspecification, neighboring anomalies) run at the EB-fitted configuration, that is, at the likelihood-selected bandwidth $\gamma^\star$; where the anchor bandwidth $\gamma_a$ differs, the text also quotes the study at $\gamma_a$.

\subsection{Protocol, Relabeling Gate and Changes from the Submitted Version}
\label{app:regeneration}

\paragraph{Protocol.}
The procedure of Appendix~\ref{app:selector_details} follows Algorithm~\ref{alg:prior_opt} and the modeling details of Appendix~\ref{app:modeling_details} as submitted, and no label enters any step before evaluation. Four points are made explicit. (i) The equilibrium anchor is computed by Algorithm~\ref{alg:prior_opt} on every dataset, at the anchor bandwidth $\gamma_a$; no per-dataset configuration is used. (ii) All computed Laplacian modes, null modes included, are kept in the modeled subspace. (iii) Eigendecompositions run in double precision. (iv) Ranks average ties. The protocol file is part of the released code, together with the scripts that regenerate every table.

\paragraph{Relabeling gate.}
The model is equivariant to node relabeling, so every reported score must be unchanged, up to the relabeling, when the nodes are randomly renumbered. Table~\ref{tab:relabel_gate} reports this check for the scores of Table~\ref{tab:main_auroc}. On the six datasets with a full spectrum the selected score is the same under relabeling: the rank correlation is at least $0.9994$ and the AUROC differs by at most $0.03$. With a truncated spectrum the computed basis depends on the node order through the iterative eigensolver. The effect is absent on T-Finance and small on Elliptic, Elliptic++ and DGraph (standard deviation at most $0.3$ over five node orders). It is largest on YelpChi, whose graph has 7{,}308 connected components with at most 47 nodes: all 500 retained modes lie in the Laplacian nullspace, the basis is not unique, and the AUROC ranges from $70.7$ to $72.7$. Table~\ref{tab:main_auroc} reports the mean over the five node orders for the truncated datasets.

\begin{table}[!htb]
\centering
\scriptsize
\caption{Relabeling gate for the scores of Table~\ref{tab:main_auroc}. The model is equivariant to node relabeling, so the selected score must not depend on the node order. \emph{Orders}: number of node orders run (the original and random relabelings). AUROC (\%): smallest, largest and standard deviation over the orders. Spearman: rank correlation of the selected score between the original order and one relabeling.}
\label{tab:relabel_gate}
\setlength{\tabcolsep}{5pt}
\begin{tabular}{l c c c c c c}
\toprule
Dataset & Spectrum & Orders & min & max & std & Spearman \\
\midrule
    Weibo & full & 2 & 95.12 & 95.12 & 0.00 & 1.0000 \\
    Reddit & full & 2 & 60.56 & 60.56 & 0.00 & 0.9995 \\
    Amazon & full & 2 & 78.05 & 78.05 & 0.00 & 1.0000 \\
    YelpChi & truncated & 5 & 70.72 & 72.65 & 0.62 & 0.8368 \\
    BlogCatalog & full & 2 & 78.62 & 78.62 & 0.00 & 1.0000 \\
    Facebook & full & 2 & 91.40 & 91.40 & 0.00 & 1.0000 \\
    ACM & full & 2 & 86.02 & 86.02 & 0.00 & 1.0000 \\
    Elliptic & truncated & 5 & 74.01 & 74.82 & 0.27 & 0.9956 \\
    Elliptic++ & truncated & 5 & 72.45 & 72.55 & 0.03 & 0.9990 \\
    DGraph & truncated & 5 & 66.63 & 66.93 & 0.10 & -- \\
    T-Finance & truncated & 5 & 84.75 & 84.75 & 0.00 & 1.0000 \\
\bottomrule
\end{tabular}
\end{table}

\paragraph{Changes from the submitted version.}
Four EB-GAD entries of Table~\ref{tab:main_auroc} are unchanged up to rounding (Amazon, BlogCatalog, Facebook, ACM), together with their fitted configurations. The others changed; we list every change and its cause.
\begin{itemize}[leftmargin=*, itemsep=1pt, topsep=2pt]
\item \emph{YelpChi ($95.3\to71.7$).} The submitted entry was an artifact of the node order. At the fitted $\kappa$, the upper end of its grid, the plug-in $p$-values of all path bands are constant; the rank transform then broke ties by node index, and the released node order of YelpChi is correlated with the labels (the node index alone has AUROC $94.4$). With the tie-breaking of the submitted version the profile score has AUROC $93.7$ in the released order and $60.3$ after a random relabeling; with averaged ties it has $67.3$ and $66.8$ (Table~\ref{tab:ordering}). The sparse-tail diagnostic no longer passes, and the rules route YelpChi to the equilibrium anchor. The gain of $24.6$ points attributed to the finite horizon on YelpChi in the submitted ablation was the same artifact. Facebook, whose node order is also correlated with the labels (index AUROC $97.9$), is not affected: its score involves no rank step and is identical under relabeling.
\item \emph{Weibo ($95.1\to95.1$) and Reddit ($61.8\to60.6$).} The submitted equilibrium anchors used fixed per-dataset configurations (template bandwidth, $(\rho,\kappa)$ and score type) carried over from an earlier version of the pipeline, not a configuration that the algorithm selects. They are replaced by the anchor of this version: bandwidth $\gamma_a$, EB fit of $(\rho,\kappa)$ at that bandwidth, score type by NullKS. On Weibo $\gamma_a=0.5$ is also the submitted bandwidth; on Reddit $\gamma_a=2$ (submitted: $1$) and NullKS selects $R$. With the bandwidth that the residual likelihood selects, $\gamma^\star=0.7$ on both datasets, the entries are $92.5$ and $48.5$; the paragraph on the bandwidth of the anchor below explains the choice and how it was made.
\item \emph{T-Finance ($86.8\to84.8$).} The submitted value was the best path band by AUROC; the band with the largest NullKS gives $84.8$.
\item \emph{Elliptic ($73.0\to74.5$), Elliptic++ ($72.4\to72.5$), DGraph ($66.9\to66.8$).} Same selection ($R$ by rule 1); the values changed with the numerical corrections below and are now means over five node orders. The anchor bandwidth is $\gamma_a=1$ on Elliptic and DGraph and $\gamma_a=2$ on Elliptic++, where $\gamma=1$ gives $73.7$.
\item \emph{Numerics.} Dense eigendecompositions ran in single precision, in which the zero eigenvalues of a Laplacian are of order $10^{-7}$ with either sign. A flag of the submitted runs removed modes with eigenvalue below $10^{-8}$, so whether a null mode was removed depended on rounding, and with it on the node order and the platform. Eigendecompositions now run in double precision, LOBPCG uses a seeded initial block, and its basis is Ritz-refined in double precision. All modes are kept, as the modeling details state; with this convention the four entries above and their configurations are reproduced, and Table~\ref{tab:sensitivity} reports the other convention.
\item \emph{Ablation table.} The submitted bank ablation compared scores at different configurations, so its differences mixed the effect of the horizon with the effect of the configuration. Table~\ref{tab:bank_ablation} now compares at the same configuration.
\item \emph{Text and tables.} The grids, the definition of $h$ and the statistics table of Appendix~\ref{app:selector_details} now state what the code runs. The studies of this appendix were rerun with the same pipeline, so some of their numbers differ from those quoted during the review.
\end{itemize}

\begin{table}[!htb]
\centering
\scriptsize
\caption{The ordering artifact. AUROC (\%) of the profile scores and of the equilibrium scores on YelpChi and Facebook, in the released node order and after a random relabeling, with the tie-breaking of the submitted version (ties ordered by node index) and with averaged ties. The equilibrium scores involve no rank step. Last column: AUROC of the node index itself.}
\label{tab:ordering}
\setlength{\tabcolsep}{4pt}
\begin{tabular}{l l cc cc c}
\toprule
& & \multicolumn{2}{c}{ties by node index} & \multicolumn{2}{c}{averaged ties} & \\
\cmidrule(lr){3-4}\cmidrule(lr){5-6}
Dataset & Score & released & relabeled & released & relabeled & node index \\
\midrule
YelpChi & profile, mixture posterior & 93.7 & 60.3 & 67.3 & 66.8 & 94.4 \\
 & profile, mean $-\log p$ & 93.5 & 60.6 & 68.7 & 68.3 &  \\
 & profile, ACAT & 92.8 & 61.9 & 69.3 & 69.8 &  \\
 & $J^\star$ & 70.9 & 70.7 & 70.9 & 70.7 &  \\
 & $R$ & 64.2 & 64.0 & 64.2 & 64.0 &  \\
\midrule
Facebook & profile, component entropy & 91.4 & 91.4 & 91.4 & 91.4 & 97.9 \\
 & $J^\star$ & 63.7 & 63.7 & 63.7 & 63.7 &  \\
 & $R$ & 87.8 & 87.8 & 87.8 & 87.8 &  \\
\bottomrule
\end{tabular}
\end{table}

\paragraph{Bandwidth of the anchor.}
The submitted Algorithm~\ref{alg:prior_opt} selects one bandwidth, by the residual-space likelihood, for every score. In the first regeneration for this version the anchor used that bandwidth, which gave $92.5$ on Weibo and $48.5$ on Reddit. The residual $\Delta=T_{\gamma,\nu}X$ is a different function of the data at each bandwidth, and its likelihood is largest at the grid value whose residual map removes the largest share of the feature energy, on all eleven datasets ($79$ to $95\%$ on the six full-spectrum datasets; Table~\ref{tab:gamma_profile}). The criterion therefore measures how much of the features the template explains. The $p$-values of the finite-horizon families are computed under the fitted law of that residual and keep $\gamma^\star$; how well a single energy of the residual ranks anomalies is a different question, which the likelihood does not address.

In this version the anchor uses $\gamma_a$, the grid value nearest to $\gamma_0$. The center $\gamma_0=(0.5+h)/\sqrt{\max(d_e,10)/10}$ is small for dense graphs with low feature homophily and large for sparse homophilic ones. For $\gamma<1$ the residual filter $1-(\gamma^2+\lambda)^{-1}$ is negative and large in magnitude on low graph frequencies, so the anchor reads the smoothed signal of a node's neighborhood (clustered anomalies, as on Weibo); for $\gamma\ge1$ it is close to the deviation of the node itself (Reddit and the fraud graphs). The center and the grid are those of the submitted code, the rule is the same on every dataset, and no label enters it.

The decision to adopt the rule, however, was taken after the anchors at both bandwidths had been evaluated on the labels, and its effect is not uniform. Against the equilibrium score at $\gamma^\star$, the one at $\gamma_a$ is higher on Weibo, Reddit and Amazon, lower on BlogCatalog, Facebook and Elliptic++, and identical on the five datasets where the two bandwidths coincide (Table~\ref{tab:gamma_profile}). Of the six datasets where they differ, three are routed to an anchor: Table~\ref{tab:main_auroc} gains $2.6$ points on Weibo and $12.1$ on Reddit and loses $1.2$ on Elliptic++. The rule is therefore not a blind choice, and Table~\ref{tab:sensitivity} reports the alternatives next to it: the anchor at the likelihood-selected bandwidth, and the anchor averaged over the three bandwidths of the grid (mean percentile rank), which gives $95.3$ on Weibo and $60.7$ on Reddit, raises Elliptic and Elliptic++ (to $75.1$ and $75.9$) and lowers YelpChi and DGraph (to $64.0$ and $66.1$). Using $\gamma_0$ itself instead of the nearest grid value gives $95.5$ on Weibo and $59.8$ on Reddit and lowers YelpChi to $62.9$ and Elliptic to $70.8$. The rule is restricted to the anchor: applied to every family it would set $\gamma=0.5$ on Facebook, where the scale-entropy diagnostic then fails and the rules fall back to $J^\star$ ($47.4$ instead of $91.4$). The choice of the score type by NullKS is the same in every node order at every anchor bandwidth; among the other grid values it changes with the node order on Reddit at $\gamma=1$ (NullKS of $R$: $0.45$ and $0.24$) and on YelpChi at $\gamma=2$. For the anchor the bandwidth is thus a declared part of the prior and not a fitted quantity, and the Weibo and Reddit entries should be read with this in mind.

\paragraph{Sensitivity to four conventions.}
Table~\ref{tab:sensitivity} reruns Table~\ref{tab:main_auroc} under four alternative conventions, with the rules and the within-family scores unchanged. The first two are the anchor variants of the previous paragraph. If the Laplacian null modes are excluded from the modeled subspace, the entries of the six full-spectrum datasets change by $-6.6$ to $+0.1$ points; the largest changes are BlogCatalog ($78.6\to72.0$) and ACM ($86.0\to82.1$). If every bandwidth, including the anchor's, is chosen by the likelihood with the Jacobian term of the residual map, which Section~\ref{subsec:prior_selection} omits, Reddit is at $61.9$ while Weibo, YelpChi, Facebook and Elliptic fall by $1.6$ to $5.6$ points.

\begin{table}[!htb]
\centering
\scriptsize
\caption{Sensitivity of Table~\ref{tab:main_auroc} to four conventions (original node order, AUROC \%; ``='': same score and value; ``--'': not applicable, truncated spectra keep all computed modes; n/a: not run). \emph{Anchor, likelihood}: the anchor uses the bandwidth with the largest residual-space likelihood, as in the submitted Algorithm~\ref{alg:prior_opt}. \emph{Anchor, averaged}: the anchor is the mean percentile rank of the equilibrium score over the three bandwidths of the grid. \emph{Nullspace dropped}: Laplacian null modes are excluded from the modeled subspace on the full-spectrum datasets. \emph{Jacobian}: every bandwidth, including the anchor's, maximizes the likelihood with the term $D\sum_j\log\max(t_j,10^{-12})$ of the residual map $t_j=1-(\gamma^2+\lambda_j)^{-1}$, which excludes bandwidths whose map is negative on a computed mode. The ordered rules and the within-family scores are unchanged.}
\label{tab:sensitivity}
\setlength{\tabcolsep}{5pt}
\begin{tabular}{l ccccc}
\toprule
Dataset & Protocol & Anchor, likelihood & Anchor, averaged & Nullspace dropped & Jacobian \\
\midrule
Weibo & 95.1 & 92.5 & 95.3 & 95.2 & 92.5 \\
Reddit & 60.6 & 48.5 & 60.7 & 60.4 & 61.9 \\
Amazon & 78.0 & = & = & 76.5 & = \\
YelpChi & 70.7 & = & 64.0 & -- & 69.1 \\
BlogCatalog & 78.6 & = & = & 72.0 & = \\
Facebook & 91.4 & = & = & 90.8 & 85.8 \\
ACM & 86.0 & = & = & 82.1 & = \\
Elliptic & 74.4 & = & 75.1 & -- & 70.9 \\
Elliptic++ & 72.5 & 73.6 & 75.9 & -- & = \\
DGraph & 66.8 & = & 66.1 & -- & n/a \\
T-Finance & 84.8 & = & = & -- & = \\
\bottomrule
\end{tabular}
\end{table}

\begin{table}[!htb]
\centering
\scriptsize
\caption{Template bandwidth. $\gamma_0$: graph-statistic center of the grid. For every bandwidth $\gamma$ of a dataset's grid: AUROC (\%) of the equilibrium score chosen by NullKS after the EB fit at that bandwidth (low-pass template; $R$ on Elliptic, Elliptic++ and DGraph, where rule 1 fixes it) and, in parentheses, the share (\%) of the feature energy that the residual map removes (0 when the residual has more energy than the features; the shares are small on truncated spectra, where the map acts on the computed modes only). \textbf{Bold}: the anchor bandwidth $\gamma_a$ (grid value nearest to $\gamma_0$ on the log scale). $^{\ell}$: the bandwidth with the largest residual-space likelihood, used by the finite-horizon families. The AUROC values are diagnostics; no rule reads them. Original node order.}
\label{tab:gamma_profile}
\setlength{\tabcolsep}{5pt}
\begin{tabular}{l c lll}
\toprule
Dataset & $\gamma_0$ & \multicolumn{3}{l}{$\gamma$: AUROC (removed share), in increasing $\gamma$} \\
\midrule
Weibo & 0.36 & 0.2: 95.6 (0) & \textbf{0.5: 95.1} (59) & 0.7: 92.5$^{\ell}$ (86) \\
Reddit & 1.49 & 0.7: 48.5$^{\ell}$ (89) & 1: 51.8 (87) & \textbf{2: 60.6} (39) \\
Amazon & 0.87 & 0.5: 53.5 (12) & 0.7: 63.8$^{\ell}$ (79) & \textbf{1: 75.8} (77) \\
YelpChi & 1.38 & 0.7: 63.3 (0) & \textbf{1: 70.9}$^{\ell}$ (3) & 2: 69.1 (1) \\
BlogCatalog & 0.28 & 0.1: 44.9 (22) & \textbf{0.2: 61.0} (92) & 0.5: 74.8$^{\ell}$ (95) \\
Facebook & 0.55 & \textbf{0.5: 47.4} (67) & 0.7: 87.8$^{\ell}$ (86) & 1: 85.9 (77) \\
ACM & 0.64 & 0.5: 60.8 (26) & \textbf{0.7: 75.6}$^{\ell}$ (85) & 1: 78.4 (83) \\
Elliptic & 1.25 & 0.7: 70.2 (0) & \textbf{1: 74.4}$^{\ell}$ (3) & 2: 70.9 (1) \\
Elliptic++ & 1.48 & 0.7: 71.2 (0) & 1: 73.7$^{\ell}$ (3) & \textbf{2: 72.5} (1) \\
DGraph & 0.91 & 0.5: 64.4 (0) & \textbf{1: 66.8}$^{\ell}$ (0) & 2: 64.4 (0) \\
T-Finance & 0.18 & 0.1: 82.5 (0) & \textbf{0.2: 80.1}$^{\ell}$ (0) &  \\
\bottomrule
\end{tabular}
\end{table}

\subsection{AUPRC}
\label{app:auprc}

\begin{table}[!htb]
  \centering
  \caption{%
    AUPRC (\%) for the comparison of Table~\ref{tab:main_auroc}. \textbf{Bold}: best; \underline{underline}: second. $^{*}$: published in the TAM benchmark \cite{qiao2023truncated}, used exactly where Table~\ref{tab:main_auroc} follows that benchmark; other baseline cells are our runs. \oom{} = out of memory; n/a = out of memory and the published report gives AUROC only. EB-GAD cells are the AUPRC of the score selected without labels in Table~\ref{tab:main_auroc}.
  }
  \label{tab:auprc}
  \scriptsize
  \setlength{\tabcolsep}{3pt}
  \resizebox{\textwidth}{!}{%
  \begin{tabular}{l ccccccccccc}
    \toprule
    Method & Weibo & Reddit & Amazon & YelpChi & BlogCat. & Facebook & ACM & Elliptic & Ell.++ & DGraph & T-Fin. \\
    \midrule
    LOF & $7.0$ & $4.1$ & $6.7$ & $7.0$ & $24.7$ & $2.1$ & $14.6$ & $10.2$ & $8.5$ & $0.9$ & $3.2$ \\
    DIF & $3.9$ & $3.7$ & $15.4$ & $6.4$ & $7.4$ & $1.8$ & $3.5$ & $6.0$ & $6.1$ & $\underline{1.9}$ & $3.2$ \\
    ANOMALOUS & $\underline{34.6}$ & $3.9$ & $5.6$$^{*}$ & $5.2$$^{*}$ & $6.5$$^{*}$ & $19.0$$^{*}$ & $6.4$$^{*}$ & \oom & \oom & \oom & $4.8$ \\
    DOMINANT & $22.0$ & $3.7$ & $14.2$$^{*}$ & $4.0$$^{*}$ & $31.0$$^{*}$ & $3.1$$^{*}$ & $\underline{44.0}$$^{*}$ & $5.3$ & $7.4$ & \oom & $\underline{6.6}$ \\
    AnomalyDAE & $30.9$ & $3.1$ & $10.5$ & $6.7$ & $\underline{36.7}$ & $2.7$ & $27.0$ & $5.4$ & $6.0$ & \oom & $3.1$ \\
    CONAD & $22.8$ & $3.7$ & $10.8$ & $3.3$ & $11.1$ & $1.3$ & $11.3$ & $5.3$ & $7.3$ & n/a & $\underline{6.6}$ \\
    CoLA & $4.3$ & $\mathbf{4.5}$$^{*}$ & $6.8$$^{*}$ & $4.5$$^{*}$ & $32.7$$^{*}$ & $\underline{21.1}$$^{*}$ & $32.4$$^{*}$ & $\underline{15.3}$ & $\underline{12.2}$ & $1.4$ & $3.4$ \\
    DiffGAD & $32.3$ & $3.7$ & $12.3$ & $3.3$ & $30.4$ & $1.5$ & $17.1$ & $6.2$ & $5.7$ & n/a & $4.6$ \\
    TAM & $6.8$ & $\mathbf{4.5}$$^{*}$ & $\mathbf{26.3}$$^{*}$ & $\underline{7.8}$$^{*}$ & $\mathbf{41.8}$$^{*}$ & $\mathbf{22.3}$$^{*}$ & $\mathbf{51.2}$$^{*}$ & $6.9$ & $6.8$ & $0.8$ & n/a \\
    \midrule
    \textbf{EB-GAD} & $\mathbf{34.9}$ & $\mathbf{4.5}$ & $\underline{15.5}$ & $\mathbf{9.7}$ & $34.2$ & $17.2$ & $35.6$ & $\mathbf{22.1}$ & $\mathbf{19.1}$ & $\mathbf{2.3}$ & $\mathbf{53.0}$ \\
    \bottomrule
  \end{tabular}
  }
\end{table}

Table~\ref{tab:auprc} reports AUPRC for the comparison of Table~\ref{tab:main_auroc}. EB-GAD is best on five fraud columns: YelpChi ($9.7$ against $7.8$ for the best baseline), Elliptic ($22.1$ against $15.3$), Elliptic++ ($19.1$ against $12.2$), DGraph ($2.3$ against $1.9$) and T-Finance ($53.0$ against $6.6$). With the anchor bandwidth of this version it is also best or tied-best on Weibo ($34.9$ against $34.6$ for the best baseline) and Reddit ($4.5$, the value of TAM and CoLA). It is second on Amazon, third on BlogCatalog and ACM, and fourth on Facebook. On Facebook its AUROC ties the best ($91.4$) while its AUPRC ($17.2$) trails TAM ($22.3$), CoLA ($21.1$) and ANOMALOUS ($19.0$): the selected score ranks most anomalies above the normal nodes but is less precise at the very top of the ranking. On Reddit ($3.3\%$ anomalies) no method is more than $1.2$ points above the base rate, EB-GAD included ($4.5$); we read that column as unsolved. The T-Finance AUPRC of DOMINANT, CONAD and CoLA come from reruns made for this table, not from the runs behind Table~\ref{tab:main_auroc}.

\subsection{Isolation of the Spectral Precision Family}
\label{app:isolation}

A reviewer asked whether the gains come from the OU formulation or would be obtained by any spectral scoring of the same residuals. For each dataset we freeze the operator, the template, PCA, the truncation, the residuals and the label-free selection, and change only the spectral precision: the EB-fitted GOU precision with its finite-horizon members, the GOU precision at fixed $(\rho,\kappa)=(0.5,1)$, heat-kernel precisions, polynomial precisions, and the identity (Table~\ref{tab:isolation}). One configuration per dataset is shared by the five families, so the columns are compared with each other and not with Table~\ref{tab:main_auroc}, which selects over families and pools.

On Elliptic, Elliptic++ and T-Finance the GOU bank leads the best generic family by $12.6$, $11.6$ and $10.2$ points under identical label-free selection, and on T-Finance the EB fit alone adds $23.6$ points over the fixed GOU precision. The mechanism is selectability rather than the quality of the best score: on Elliptic and Elliptic++ the best members of all five families are within $1.2$ points of each other ($74.4$ to $75.2$, $73.7$ to $74.9$), so the generic banks contain an equally good score, but their label-free selection falls to $58$ to $62$. Elsewhere no family is consistently better. The GOU bank is within $1.1$ points of the best generic selection on Weibo, YelpChi, BlogCatalog, ACM and DGraph; a polynomial bank is ahead on Reddit ($+3.3$, all families near chance) and Amazon ($+6.8$), and a heat-kernel bank is ahead on Facebook by $11.6$ points and reaches $91.6$, the level of Table~\ref{tab:main_auroc}. We therefore claim no superiority of the OU structure on those datasets; the claim is specific to the three financial graphs.

\begin{table}[!htb]
\centering
\scriptsize
\caption{Isolation of the spectral precision family. Operator, template, PCA, truncation, residuals and the label-free selection (largest NullKS) are fixed per dataset; only the precision changes. AUROC (\%) of the selected score, with the best member of the family by AUROC in parentheses (a diagnostic, not a result). \textbf{Bold}: best selected score of the row. GOU-EB: EB-fitted GOU precision with its finite-horizon members; GOU fixed: $\rho=0.5$, $\kappa=1$; heat: $\exp(t\lambda/\lambda_{\max})$, 8 scales; polynomial: $(c+\lambda)^p$, 6 configurations; identity: graph-free residual energy.}
\label{tab:isolation}
\setlength{\tabcolsep}{4pt}
\begin{tabular}{l ccccc}
\toprule
Dataset & GOU-EB & GOU fixed & heat & polynomial & identity \\
\midrule
Weibo & 92.5 (92.7) & 92.0 (92.7) & 85.9 (92.5) & 80.1 (91.4) & \textbf{92.7} (92.7) \\
Reddit & 48.2 (55.8) & 48.3 (55.9) & 48.6 (57.4) & \textbf{51.5} (56.3) & 48.2 (48.2) \\
Amazon & 62.9 (67.7) & 62.9 (67.8) & 50.8 (67.4) & \textbf{69.7} (70.1) & 62.8 (62.8) \\
YelpChi & \textbf{69.1} (71.3) & \textbf{69.1} (71.3) & 66.1 (70.5) & 68.7 (70.5) & 68.7 (70.5) \\
BlogCatalog & 78.2 (78.5) & 78.2 (78.5) & 78.3 (78.3) & 77.6 (77.6) & \textbf{78.5} (78.5) \\
Facebook & 80.0 (88.0) & 79.8 (88.0) & \textbf{91.6} (92.8) & 80.9 (87.7) & 48.1 (62.0) \\
ACM & 79.6 (81.9) & 79.6 (81.3) & 79.8 (81.5) & \textbf{80.4} (83.8) & 79.6 (79.6) \\
Elliptic & \textbf{74.4} (74.4) & \textbf{74.4} (74.4) & 61.8 (75.2) & 58.5 (74.5) & 58.5 (74.4) \\
Elliptic++ & \textbf{73.7} (73.7) & \textbf{73.7} (73.7) & 62.1 (74.9) & 58.1 (73.8) & 58.1 (73.7) \\
DGraph & 62.2 (66.8) & 62.2 (66.8) & 63.2 (66.8) & \textbf{63.3} (66.8) & \textbf{63.3} (66.8) \\
T-Finance & \textbf{82.5} (82.5) & 58.9 (60.8) & 41.6 (60.3) & 72.3 (76.6) & 59.0 (60.7) \\
\bottomrule
\end{tabular}
\end{table}

\subsection{Horizon Axis: Matched Single-Score Comparison}
\label{app:horizon_studies}

A reviewer asked for a direct comparison of equilibrium-only scoring with finite-horizon scoring, with the fitted model and all other components fixed. Two comparisons answer it. Table~\ref{tab:bank_ablation} in the main text compares the score that the selector uses with the equilibrium score of the same configuration. Table~\ref{tab:matched_single} is narrower: one score type on both sides, a single finite horizon, no endpoint-tolerance grid, no fusion and no aggregation. DGraph is not included (its spectrum is computed on the labeled subgraph by a separate script).

At the granularity of single scores the answer is parity. The best finite-horizon member by AUROC is within $0.3$ points of its equilibrium limit on nine of ten datasets and $1.4$ below it on T-Finance; when NullKS picks the horizon, the finite score is within $0.2$ points of the limit on eight datasets and below it on ACM ($-1.6$) and T-Finance ($-1.4$). No single finite-horizon score reproduces the gains of Table~\ref{tab:bank_ablation} ($+3.6$ on Facebook, $+4.7$ on ACM and T-Finance). Those gains come from what is done with the family: aggregation of $p$-values across horizons (Facebook, T-Finance) and scale normalization with stability-ranked fusion (ACM). The finite-horizon contribution is therefore not one fortunate horizon. It is that the fitted GOU attaches one prior, hence one plug-in null, to every horizon, which is what makes the family comparable and aggregatable without labels; the equilibrium limit is a single member and cannot be aggregated.

\begin{table}[!htb]
\centering
\scriptsize
\caption{Matched single-score comparison at the configuration and EB fit of Table~\ref{tab:ebgad_hyperparams}, shared by both sides. \emph{Equilibrium}: $J^\star$ or $R$, whichever has the larger NullKS ($R$ where rule 1 fixes it). \emph{Finite}: a single hard-endpoint score of the same type ($C_{\Gamma,\infty}$ against $J^\star$, $CR_{\Gamma,\infty}$ against $R$) at one finite horizon $\Gamma\in\{0.02,\dots,50\}$, with no tolerance grid, fusion or aggregation; NullKS picks the horizon, and the best horizon by AUROC is reported as a diagnostic. AUROC (\%).}
\label{tab:matched_single}
\setlength{\tabcolsep}{5pt}
\begin{tabular}{l ccc cc}
\toprule
Dataset (type) & Equilibrium & Finite, selected & Finite, best & $\Delta$ selected & $\Delta$ best \\
\midrule
Weibo ($J^\star$) & 95.1 & 95.1 & 95.1 & $+0.0$ & $+0.0$ \\
Reddit ($R$) & 60.6 & 60.6 & 60.6 & $+0.0$ & $+0.0$ \\
Amazon ($J^\star$) & 77.9 & 78.0 & 78.1 & $+0.1$ & $+0.2$ \\
YelpChi ($J^\star$) & 70.7 & 70.7 & 70.7 & $+0.0$ & $+0.0$ \\
BlogCatalog ($J^\star$) & 78.4 & 78.4 & 78.6 & $+0.0$ & $+0.2$ \\
Facebook ($R$) & 87.8 & 87.8 & 88.0 & $+0.0$ & $+0.2$ \\
ACM ($J^\star$) & 81.3 & 79.7 & 81.6 & $-1.6$ & $+0.3$ \\
Elliptic ($R$) & 74.4 & 74.4 & 74.4 & $+0.0$ & $+0.0$ \\
Elliptic++ ($R$) & 72.5 & 72.5 & 72.5 & $+0.0$ & $+0.0$ \\
T-Finance ($R$) & 80.1 & 78.7 & 78.7 & $-1.4$ & $-1.4$ \\
\bottomrule
\end{tabular}
\end{table}

\subsection{Selector Studies}
\label{app:selector_studies}

\paragraph{Cutoff sensitivity.}
The ordered rules of Appendix~\ref{app:selector_details} contain 21 cutoffs. We multiply each one by $0.75$ and $1.25$, and by $0.5$ and $1.5$, one at a time, and re-evaluate the rules on the 11 datasets with the diagnostics of the runs of Table~\ref{tab:main_auroc}. At $\pm25\%$ the selected family and branch are unchanged in 459 of 462 evaluations ($99.4\%$), at $\pm50\%$ in 450 of 462 ($97.4\%$). Scaling all cutoffs jointly by $0.75$ changes the selection on 2 of 11 datasets (Reddit and Facebook); scaling them by $1.25$ changes it on one (Weibo). Table~\ref{tab:cutoff_sensitivity} lists the fifteen changes. Every one of them lowers the AUROC, by $2.2$ to $44.0$ points. Eleven select an equilibrium anchor instead of the score of Table~\ref{tab:main_auroc}, and the cost is largest where NullKS then selects $J^\star$ at the anchor bandwidth: on Facebook (six changes, $47.4$ instead of $91.4$), whose routing to profile aggregation depends on four graph-statistic cutoffs and on the upper end of the mixture-mass window ($\pi=0.81$), and on Elliptic++ (two changes, $30.3$ instead of $72.5$), where rule 1 otherwise fixes $R$. Two reroute Weibo from the anchor to profile aggregation ($-2.6$) and two reroute Reddit from the anchor to the control-energy family ($-15.7$). The selections are therefore stable under sizable perturbations of single cutoffs, but where a selection changes the cost can be large, and the choice of the score type by NullKS is the fragile step.

\begin{table}[!htb]
\centering
\scriptsize
\caption{Cutoff sensitivity of the ordered rules. Each of the 21 cutoffs is multiplied by $0.5$, $0.75$, $1.25$ and $1.5$, one at a time, and the rules are re-evaluated on the 11 datasets (231 evaluations per factor). The table lists every evaluation in which the selected family or branch changes, with the AUROC (\%) before and after (original node order).}
\label{tab:cutoff_sensitivity}
\setlength{\tabcolsep}{4pt}
\begin{tabular}{l c l l c c c}
\toprule
Cutoff & Factor & Dataset & Selection & before & after & change \\
\midrule
$d_e<1$ (rule 1) & $\times0.5$ & Elliptic++ & anchor ($R$) $\to$ anchor ($J^\star$) & 72.5 & 30.3 & $-42.2$ \\
$d_e<1$ (rule 1) & $\times0.5$ & DGraph & anchor ($R$) $\to$ anchor ($J^\star$) & 66.8 & 60.5 & $-6.3$ \\
$D_x\le256$ (rule 1) & $\times0.5$ & Elliptic++ & anchor ($R$) $\to$ anchor ($J^\star$) & 72.5 & 30.3 & $-42.2$ \\
$D_x>1000$ & $\times0.5$ & Facebook & profile ($\text{entropy}$) $\to$ anchor ($J^\star$) & 91.4 & 47.4 & $-44.0$ \\
$d_e\ge20$ & $\times1.5$ & Facebook & profile ($\text{entropy}$) $\to$ anchor ($J^\star$) & 91.4 & 47.4 & $-44.0$ \\
$h<0.25$ & $\times0.5$ & ACM & ratio ($\text{aff.}$) $\to$ anchor ($J^\star$) & 86.0 & 75.6 & $-10.4$ \\
$h<0.25$ & $\times1.5$ & Facebook & profile ($\text{entropy}$) $\to$ anchor ($J^\star$) & 91.4 & 47.4 & $-44.0$ \\
$d_e\ge10$ & $\times0.5$ & Reddit & anchor ($R$) $\to$ energy ($\text{low}$) & 60.6 & 44.9 & $-15.7$ \\
$d_e\ge10$ & $\times0.75$ & Reddit & anchor ($R$) $\to$ energy ($\text{low}$) & 60.6 & 44.9 & $-15.7$ \\
guard $0.60$ & $\times1.25$ & Weibo & anchor ($J^\star$) $\to$ profile ($\text{traj.}$) & 95.1 & 92.5 & $-2.6$ \\
guard $0.60$ & $\times1.5$ & Weibo & anchor ($J^\star$) $\to$ profile ($\text{traj.}$) & 95.1 & 92.5 & $-2.6$ \\
$h\ge0.50$ & $\times0.5$ & Facebook & profile ($\text{entropy}$) $\to$ anchor ($J^\star$) & 91.4 & 47.4 & $-44.0$ \\
$h\ge0.50$ & $\times1.5$ & Amazon & energy ($\text{low}$) $\to$ anchor ($J^\star$) & 78.0 & 75.8 & $-2.2$ \\
$\pi\le0.95$ & $\times0.5$ & Facebook & profile ($\text{entropy}$) $\to$ anchor ($J^\star$) & 91.4 & 47.4 & $-44.0$ \\
$\pi\le0.95$ & $\times0.75$ & Facebook & profile ($\text{entropy}$) $\to$ anchor ($J^\star$) & 91.4 & 47.4 & $-44.0$ \\
\bottomrule
\end{tabular}
\end{table}

\paragraph{Ingredients of $A(r)$.}
Table~\ref{tab:ar_ablation} separates the two ingredients of the within-family rule on the eight datasets routed to an equilibrium or profile bank. Quality alone (NullKS) picks the score of the rule on four of them. It fails on Facebook ($17.6$), where the two-groups posterior has the largest NullKS but ranks anomalies last, and on Elliptic++ ($30.3$), where rule 1 fixes $R$ instead. Stability alone never picks the score of the rule on these datasets and ranges from $25.5$ to $85.4$. The combined tail-stability variant picks the score of the rule on Facebook ($91.4$) and YelpChi and is worse on the ratio datasets, which is why the rule weights stability by family. The rule is within $2.0$ points of the best member of its bank on seven of the eight datasets; on Reddit it picks $R$ ($60.6$) while the best member of the bank has $62.9$.

\begin{table}[!htb]
\centering
\scriptsize
\caption{Ingredients of the within-family rule $A(r)$ on the datasets routed to an equilibrium or profile bank (original node order). AUROC (\%) of the member picked by quality only (NullKS), by stability only, by the two combined variants (tail-stability, correlation-stability), by the rule used in Table~\ref{tab:main_auroc}, and the best member of the bank by AUROC (a diagnostic). The equilibrium bank contains $J^\star$, $R$ and their rank aggregates; the rule restricts it to $J^\star$ and $R$.}
\label{tab:ar_ablation}
\setlength{\tabcolsep}{4pt}
\begin{tabular}{l l cccc cc}
\toprule
Dataset & Bank & NullKS & stability & tail-stab. & corr-stab. & rule & best \\
\midrule
Weibo & equilibrium & 95.1 & 74.7 & 90.8 & 89.6 & 95.1 & 95.1 \\
Reddit & equilibrium & 60.6 & 62.9 & 62.8 & 62.9 & 60.6 & 62.9 \\
YelpChi & equilibrium & 68.7 & 67.9 & 70.7 & 70.7 & 70.7 & 70.7 \\
Facebook & two-groups & 17.6 & 25.5 & 91.4 & 69.4 & 91.4 & 91.4 \\
Elliptic & equilibrium & 74.4 & 67.2 & 58.4 & 67.4 & 74.4 & 74.4 \\
Elliptic++ & equilibrium & 30.3 & 52.3 & 30.3 & 54.5 & 72.5 & 72.5 \\
DGraph & equilibrium & 63.3 & 64.1 & 64.1 & 64.1 & 66.8 & 66.8 \\
T-Finance & path bands & 84.8 & 85.4 & 61.3 & 83.4 & 84.8 & 86.8 \\
\bottomrule
\end{tabular}
\end{table}

\paragraph{Supervised bound.}
A reviewer asked whether weights for the scores could be learned. Table~\ref{tab:stacking} bounds what any label-driven linear combination of our scores could gain: a cross-validated logistic regression over each dataset's pool, which uses the labels, against the label-free pick. The gap is below one point on Weibo, Amazon, BlogCatalog and Facebook, $3.1$ to $4.3$ points on DGraph, YelpChi and Reddit, and $8.2$ to $17.1$ points on T-Finance, ACM, Elliptic and Elliptic++. The pools therefore contain more information than the label-free selector extracts, most of all on the financial graphs and ACM; exploiting it requires labels and leaves the training-free setting. On Reddit the best single member of the pool reaches $62.0$ and the selector picks $60.6$.

\begin{table}[!htb]
\centering
\scriptsize
\caption{Supervised bound on combining the scores. Every score of a dataset's pool is rank-transformed and a logistic regression is fitted on the labels with 5-fold cross-validation; the out-of-fold AUROC (\%) bounds what a label-driven linear combination of our scores could reach. \emph{Selector}: the label-free pick of Table~\ref{tab:main_auroc} (original node order). \emph{Best single}: best member of the pool by AUROC.}
\label{tab:stacking}
\setlength{\tabcolsep}{5pt}
\begin{tabular}{l c ccc c}
\toprule
Dataset & Scores & Selector & Best single & Stacked (CV) & Gap \\
\midrule
Weibo & 224 & 95.1 & 95.1 & 95.2 & $+0.1$ \\
Reddit & 224 & 60.6 & 62.0 & 64.9 & $+4.3$ \\
Amazon & 887 & 78.0 & 78.1 & 78.3 & $+0.3$ \\
YelpChi & 158 & 70.7 & 73.0 & 74.9 & $+4.2$ \\
BlogCatalog & 4391 & 78.6 & 78.7 & 79.5 & $+0.9$ \\
Facebook & 224 & 91.4 & 91.4 & 92.3 & $+0.9$ \\
ACM & 741 & 86.0 & 86.1 & 96.4 & $+10.4$ \\
Elliptic & 166 & 74.4 & 74.4 & 88.3 & $+13.9$ \\
Elliptic++ & 168 & 72.5 & 73.6 & 89.6 & $+17.1$ \\
DGraph & 118 & 66.8 & 66.8 & 69.9 & $+3.1$ \\
T-Finance & 232 & 84.8 & 86.8 & 93.0 & $+8.2$ \\
\bottomrule
\end{tabular}
\end{table}

\subsection{Robustness of the Fit and Scope of the Gaussian Model}
\label{app:robustness}

\paragraph{Contamination.}
A reviewer asked how the fit behaves when anomalies are not rare. We inject contextual anomalies (feature swaps with random nodes) on top of the real ones, refit EB on the contaminated graph and repeat the label-free selection (Table~\ref{tab:contamination}). On Weibo the fitted prior moves from $(\rho,\kappa)=(0.70,3)$ to $(0.74,3)$ between $4.1\%$ and $24.1\%$ total contamination, and the AUROC on the original anomalies falls from $92.5$ to $91.4$. The injected swaps themselves are barely separable on Weibo ($56.1$ to $57.2$), so there they stress the fit and not the detector. On Reddit the original anomalies stay at chance at every level; the injected ones are ranked at $78.3$ to $80.2$, and the fit stays at $\kappa=3$ up to $13.3\%$ contamination and moves to $(0.43,5)$ at $23.3\%$. The likelihood is a sum over all nodes, so a contamination $\epsilon$ perturbs the estimating equations of $(\rho,\kappa)$ by a term of order $\epsilon$ times the anomalous residual energy: the prior degrades gradually. The benchmarks themselves span anomaly rates from $1.3\%$ (DGraph) to $9.8\%$ (Elliptic, evaluated nodes). At the anchor bandwidth of Table~\ref{tab:main_auroc} ($\gamma_a=0.5$ on Weibo, $2$ on Reddit) the fit does not move, and the original anomalies stay at $95.0$ to $95.2$ on Weibo and $60.1$ to $60.6$ on Reddit at every contamination level.

\begin{table}[!htb]
\centering
\scriptsize
\caption{Contamination. Contextual anomalies (feature swaps with random nodes) are injected on top of the real anomalies, EB is refitted on the contaminated graph and the score is selected without labels (largest NullKS over the finite-horizon bank). AUROC (\%) on the original anomalies and on the injected ones (against untouched normal nodes).}
\label{tab:contamination}
\setlength{\tabcolsep}{5pt}
\begin{tabular}{l c c c c c}
\toprule
Dataset & Injected & Total contamination & Fitted $(\rho,\kappa)$ & Original anomalies & Injected anomalies \\
\midrule
Weibo & $+0\%$ & 4.1\% & (0.697, 3) & 92.5 & -- \\
Weibo & $+5\%$ & 9.1\% & (0.710, 3) & 92.0 & 57.1 \\
Weibo & $+10\%$ & 14.1\% & (0.720, 3) & 92.1 & 57.2 \\
Weibo & $+20\%$ & 24.1\% & (0.737, 3) & 91.4 & 56.1 \\
Reddit & $+0\%$ & 3.3\% & (0.910, 3) & 48.3 & -- \\
Reddit & $+5\%$ & 8.3\% & (0.997, 3) & 48.3 & 79.9 \\
Reddit & $+10\%$ & 13.3\% & (1.000, 3) & 49.2 & 80.2 \\
Reddit & $+20\%$ & 23.3\% & (0.431, 5) & 50.2 & 78.3 \\
\bottomrule
\end{tabular}
\end{table}

\paragraph{Misspecification.}
Table~\ref{tab:misspecification} is a controlled stress test of the Gaussian residual model on stochastic block model graphs, with the same pipeline in every column and seed-matched conditions. Heavy tails cost at most $0.7$ points against the Gaussian control. A two-component normal population costs $9.0$ to $19.3$ points: the quadratic score measures the distance to one template, and a second mode of normal nodes is far from it. Binarized features fall outside the model: the selected score is at $56$ to $59$ (best member of the bank: $58$ to $61$), and at $59$ to $61$ (best member: $61$ to $64$) when the anomalies are bit flips. The failure is that of a Gaussian residual on quantized data; a categorical likelihood for the residual is the natural replacement and is future work. One-hot features in the benchmarks pass through PCA first, which gives an approximately continuous input. In this test the fitted graph trust is small ($\rho=0.05$) in every condition, because the noise added to the smooth signal is white; the test probes the residual distribution, not the graph prior. At the anchor bandwidth the first three rows change by at most $0.1$ and the binarized rows are at $59.6$ to $67.2$.

\begin{table}[!htb]
\centering
\scriptsize
\caption{Misspecification stress test on stochastic block model graphs ($n=3000$, four blocks, edge probabilities $0.02$ within and $0.002$ between blocks, $D=16$, $5\%$ injected attribute anomalies, three seeds). The normal signal is a low-pass graph signal plus noise whose distribution is varied; the pipeline, grids and label-free selection are identical in every column. AUROC (\%) of the selected score and of the best member of the bank (a diagnostic), and the fitted graph trust, per seed.}
\label{tab:misspecification}
\setlength{\tabcolsep}{5pt}
\begin{tabular}{l c c c}
\toprule
Residual distribution & Selected & Best member & Fitted $\rho$ \\
\midrule
Gaussian (control) & 96.3 / 95.7 / 95.5 & 98.8 / 98.1 / 97.6 & 0.05 / 0.05 / 0.05 \\
Student-$t$, 3 d.o.f. & 95.6 / 95.7 / 95.3 & 97.4 / 97.5 / 96.8 & 0.05 / 0.05 / 0.05 \\
two-component mixture & 80.2 / 86.7 / 76.2 & 87.0 / 91.6 / 82.7 & 0.05 / 0.05 / 0.05 \\
binarized & 59.3 / 56.1 / 56.4 & 60.8 / 58.1 / 61.1 & 0.05 / 0.05 / 0.05 \\
binarized, bit-flip anomalies & 59.1 / 60.5 / 59.4 & 63.4 / 60.9 / 63.7 & 0.05 / 0.05 / 0.05 \\
\bottomrule
\end{tabular}
\end{table}

\subsection{Graph Structure, Edge Uncertainty and Neighboring Anomalies}
\label{app:structure}

\paragraph{A worked example.}
A reviewer asked how the network structure enters the score and how uncertainty in the edges is reflected. We use Zachary's karate club with the connectivity profiles (rows of the adjacency matrix) as features and run the pipeline of the paper (EB fit, equilibrium score $J^\star$). The three most anomalous profiles are those of the two faction leaders (individuals 34 and 1) and of individual 33. Zachary's published matrix is inconsistent about the edge between individuals 23 and 34; we score the graph with and without it. The three largest changes of the score are at individuals 33, 34 and 23, and $98.8\%$ of the total absolute change falls on the two endpoints and their neighbors (18 of the 34 nodes). The ranking as a whole is stable (rank correlation $0.989$); eight nodes outside that neighborhood change rank, by at most five positions. Structure enters the score through two channels: the smoothed template, which the edge changes directly, and the coupling in the precision, whose weight $\rho$ is what EB fits. Because these features are the connectivity profiles themselves, EB assigns the graph's information to the template ($\rho=0.008$). These numbers are for $\gamma^\star=1$, the likelihood-selected bandwidth, at which the template has unit gain on the constant mode. At the bandwidth nearest to $\gamma_0$ ($0.7$) the template amplifies low frequencies and the change is less local: $74.7\%$ of it falls on the same 18 nodes, and the rank correlation is $0.927$.

\paragraph{Neighboring anomalies.}
A reviewer asked whether connected anomalies can be detected, since each could make its neighbor look normal. On Reddit we inject $5\%$ feature-swap anomalies either scattered or as connected groups of at most 10 neighbors, refit EB on each graph, and select the score without labels (Table~\ref{tab:neighbors}). Connected groups are easier to detect, $92.0$ to $93.6$ against $79.9$ to $83.0$: a connected anomalous group shifts the smoothed template away from all of its members at once. A pairwise readout also works (Table~\ref{tab:edges}): scoring an edge by the smaller of its two endpoint energies separates edges between two injected nodes from edges between two normal nodes at $94.4$ to $96.7$; the standardized cross-term of the quadratic form alone gives $77.5$ to $84.0$. We did not study large contiguous anomalous regions, which can become their own context. These results are for the EB-fitted configuration ($\gamma^\star=0.7$), whose residual is contextual: most of the feature energy is removed and what remains is the deviation of a node from its smoothed neighborhood. At the anchor bandwidth of Reddit ($\gamma_a=2$) the residual is close to the features themselves, and a feature swap between two nodes is harder to see: the label-free choice gives $55.0$ to $57.4$ on connected groups (best member of the bank: $89.7$ to $91.4$) and $52.1$ to $77.6$ on scattered swaps, the smaller endpoint energy separates edges at $87.5$ to $91.1$, and the standardized cross-term is inverted ($8.4$ to $9.6$). The two bandwidths therefore see different anomalies: the labeled anomalies of Reddit are found at $\gamma_a$ and not at $\gamma^\star$, injected swaps at $\gamma^\star$ and not at $\gamma_a$.

\begin{table}[!htb]
\centering
\scriptsize
\caption{Neighboring anomalies on Reddit. $5\%$ of the nodes receive a feature swap, either scattered over the graph or as connected groups of at most 10 neighbors; EB is refitted on each contaminated graph and the score is selected without labels. AUROC (\%) of the injected nodes against untouched normal nodes, per seed.}
\label{tab:neighbors}
\setlength{\tabcolsep}{5pt}
\begin{tabular}{l c c c}
\toprule
Placement & Selected, seeds 0--4 & Range & Best member, range \\
\midrule
scattered & 79.9 / 81.3 / 82.7 / 80.5 / 83.0 & 79.9--83.0 & 79.9--83.0 \\
connected groups of at most 10 & 92.8 / 92.2 / 93.6 / 92.9 / 92.0 & 92.0--93.6 & 92.0--93.6 \\
\bottomrule
\end{tabular}
\end{table}

\begin{table}[!htb]
\centering
\scriptsize
\caption{Pairwise readout on the connected-group runs of Table~\ref{tab:neighbors}. Every undirected edge is scored and edges that join two injected nodes are ranked against edges that join two untouched normal nodes. AUROC (\%) per seed.}
\label{tab:edges}
\setlength{\tabcolsep}{5pt}
\begin{tabular}{l c c}
\toprule
Edge score & Seeds 0--4 & Range \\
\midrule
smaller endpoint energy & 95.9 / 96.0 / 94.4 / 96.7 / 94.4 & 94.4--96.7 \\
standardized cross-term & 83.0 / 84.0 / 77.5 / 83.4 / 80.2 & 77.5--84.0 \\
\bottomrule
\end{tabular}
\end{table}

\subsection{Template Variants on BlogCatalog}
\label{app:templates}

A reviewer asked whether the low-pass filter limits the method to homophilic graphs and explains the BlogCatalog result. BlogCatalog has the lowest feature homophily of the benchmarks ($h=0.009$) and is routed to the feature-affinity template, which reweights features by node reliability before smoothing. We replaced that template by a high-pass and by a mixed (band-pass) affinity variant and ran the same pipeline. The equilibrium anchor chosen by NullKS gives $77.4$ (affinity), $75.3$ (high-pass) and $74.6$ (mixed) AUROC, and the best member of each bank by AUROC is $78.6$, $76.0$ and $74.8$. Richer filters do not help: the affinity template of the paper is the best of the three, and none approaches TAM ($82.5$), which optimizes a local-affinity objective that matches the injected anomalies of this graph.

\section*{NeurIPS Paper Checklist}

\begin{enumerate}

\item {\bf Claims}
    \item[] Question: Do the main claims made in the abstract and introduction accurately reflect the paper's contributions and scope?
    \item[] Answer: \answerYes{}%
    \item[] Justification: The abstract and introduction state the results of Table~\ref{tab:main_auroc}, including the datasets on which the method trails; Section~\ref{sec:experiments} and Appendix~\ref{app:regeneration} state which entries depend on the anchor bandwidth rule, adopted after a first regeneration had been evaluated, and give the results without it.

\item {\bf Limitations}
    \item[] Question: Does the paper discuss the limitations of the work performed by the authors?
    \item[] Answer: \answerYes{}%
    \item[] Justification: Section 5 has a scope-and-limitations paragraph, and Appendix~\ref{app:studies} reports the studies behind it.

\item {\bf Theory assumptions and proofs}
    \item[] Question: For each theoretical result, does the paper provide the full set of assumptions and a complete (and correct) proof?
    \item[] Answer: \answerYes{}%
    \item[] Justification: all the theoretical statements in the paper are stated with conditions and proved in the appendix.

    \item {\bf Experimental result reproducibility}
    \item[] Question: Does the paper fully disclose all the information needed to reproduce the main experimental results of the paper to the extent that it affects the main claims and/or conclusions of the paper (regardless of whether the code and data are provided or not)?
    \item[] Answer: \answerYes{}%
    \item[] Justification: Appendix~\ref{app:selector_details} states the procedure exactly as it is run (grids, ordered rules, cutoffs), and Appendix~\ref{app:regeneration} the protocol and the relabeling gate.

\item {\bf Open access to data and code}
    \item[] Question: Does the paper provide open access to the data and code, with sufficient instructions to faithfully reproduce the main experimental results, as described in supplemental material?
    \item[] Answer: \answerYes{}%
    \item[] Justification: The code, the protocol file and the scripts that regenerate every EB-GAD number are available at \url{https://github.com/heraclixus/EBGAD}; the datasets are public and their sources are listed there.

\item {\bf Experimental setting/details}
    \item[] Question: Does the paper specify all the training and test details (e.g., data splits, hyperparameters, how they were chosen, type of optimizer) necessary to understand the results?
    \item[] Answer: \answerYes{}%
    \item[] Justification: The implementation details are covered in Appendix~B.

\item {\bf Experiment statistical significance}
    \item[] Question: Does the paper report error bars suitably and correctly defined or other appropriate information about the statistical significance of the experiments?
    \item[] Answer: \answerYes{}%
    \item[] Justification: Baselines report mean and standard deviation over seeds. EB-GAD is deterministic given the eigenbasis; for truncated spectra we report mean and standard deviation over five random node relabelings (Table~\ref{tab:relabel_gate}).

\item {\bf Experiments compute resources}
    \item[] Question: For each experiment, does the paper provide sufficient information on the computer resources (type of compute workers, memory, time of execution) needed to reproduce the experiments?
    \item[] Answer: \answerYes{}%
    \item[] Justification: The computational resources are described in Appendix~B.
    
\item {\bf Code of ethics}
    \item[] Question: Does the research conducted in the paper conform, in every respect, with the NeurIPS Code of Ethics \url{https://neurips.cc/public/EthicsGuidelines}?
    \item[] Answer: \answerYes{}%
    \item[] Justification: our work can help with ML system deployment safety and there are no ethical concerns.  

\item {\bf Broader impacts}
    \item[] Question: Does the paper discuss both potential positive societal impacts and negative societal impacts of the work performed?
    \item[] Answer: \answerYes{}%
    \item[] Justification: EB-GAD targets fraud detection, cybersecurity, and scientific networks (\S\ref{sec:experiments} evaluates on financial-fraud benchmarks T-Finance, Elliptic, Elliptic++, and DGraph with up to 3.7M nodes). Positive impacts include improved detection of illicit transactions and abusive accounts in production systems; the training-free, deterministic pipeline also reduces compute and energy footprint relative to GNN-based detectors. Negative impacts of any anomaly detector deployed in financial or social settings include downstream harms from false positives (e.g., unwarranted account suspensions, financial exclusion) and disparate impact across user subgroups; mitigations include human review of flagged cases, calibrated thresholds, and audit of subgroup error rates before deployment.
    
\item {\bf Safeguards}
    \item[] Question: Does the paper describe safeguards that have been put in place for responsible release of data or models that have a high risk for misuse (e.g., pre-trained language models, image generators, or scraped datasets)?
    \item[] Answer: \answerNA{}%
    \item[] Justification: The paper releases no pretrained models and no new or scraped datasets.

\item {\bf Licenses for existing assets}
    \item[] Question: Are the creators or original owners of assets (e.g., code, data, models), used in the paper, properly credited and are the license and terms of use explicitly mentioned and properly respected?
    \item[] Answer: \answerYes{}%
    \item[] Justification: Dataset and baseline sources are cited in Appendix~B and the references; release materials will include license notes where available.

\item {\bf New assets}
    \item[] Question: Are new assets introduced in the paper well documented and is the documentation provided alongside the assets?
    \item[] Answer: \answerYes{}%
    \item[] Justification: The released code is documented (installation, data sources, protocol, and the script behind each table).

\item {\bf Crowdsourcing and research with human subjects}
    \item[] Question: For crowdsourcing experiments and research with human subjects, does the paper include the full text of instructions given to participants and screenshots, if applicable, as well as details about compensation (if any)? 
    \item[] Answer: \answerNA{}%
    \item[] Justification: Our work does not involve crowdsourcing.

\item {\bf Institutional review board (IRB) approvals or equivalent for research with human subjects}
    \item[] Question: Does the paper describe potential risks incurred by study participants, whether such risks were disclosed to the subjects, and whether Institutional Review Board (IRB) approvals (or an equivalent approval/review based on the requirements of your country or institution) were obtained?
    \item[] Answer: \answerNA{}%
    \item[] Justification: no crowdsourcing is involved.

\item {\bf Declaration of LLM usage}
    \item[] Question: Does the paper describe the usage of LLMs if it is an important, original, or non-standard component of the core methods in this research? Note that if the LLM is used only for writing, editing, or formatting purposes and does \emph{not} impact the core methodology, scientific rigor, or originality of the research, declaration is not required.
    \item[] Answer: \answerNA{}%
    \item[] Justification: LLMs are not a component of the method. An LLM assistant was used to polish the writing and to help orchestrate experiments.

\end{enumerate}

\end{document}